\documentclass[12pt, a4paper]{article}
\usepackage{amsmath,amsfonts,amsthm,mathrsfs}

\numberwithin{equation}{section}
\newtheorem{theorem}{Theorem}
\newtheorem{lemma}{Lemma}

\newtheorem{definition}{Definition}

\usepackage[scale={0.72,0.743}, hmarginratio=1:1, vmarginratio=1:1, headheight=2ex, headsep = 2ex, footskip = 3.9ex, nomarginpar]{geometry}

\def\NN{\mathbb N}
\def\ZZ{\mathbb Z}
\def\RR{\mathbb R}
\def\CC{\mathbb C}

\begin{document}
\title{Theory of Deep Convolutional Neural Networks III: Approximating Radial Functions}
\author{Tong Mao, Zhongjie Shi, and Ding-Xuan Zhou \\
School of Data Science, City University of Hong Kong \\
Kowloon, Hong Kong \\
Email: mazhou@cityu.edu.hk}
\date{}

\maketitle
\begin{abstract}
We consider a family of deep neural networks consisting of two groups of convolutional layers, a downsampling operator, and a fully connected layer.
The network structure depends on two structural parameters which determine the numbers of convolutional layers and the width of the fully connected layer.
We establish an approximation theory with explicit approximation rates when the approximated function takes a composite form $f\circ Q$ with a feature polynomial $Q$ and a univariate function $f$. In particular, we prove that such a network can outperform fully connected shallow networks
in approximating radial functions with $Q(x) =|x|^2$, when the dimension $d$ of data from $\RR^d$ is large.
This gives the first rigorous proof for the superiority of deep convolutional neural networks in approximating functions with special structures.
Then we carry out generalization analysis for empirical risk minimization with such a deep network in a regression framework with
the regression function of the form $f\circ Q$.
Our network structure which does not use any composite information or the functions $Q$ and $f$ can automatically extract features and make use of the composite nature of the regression function via tuning the structural parameters.
Our analysis provides an error bound which decreases with the network depth to a minimum and then increases,
verifying theoretically a trade-off phenomenon observed for network depths in many practical applications.
\end{abstract}


\baselineskip 16pt

\section{Introduction}

Deep learning has been a powerful tool for processing big data from many fields of science and technology \cite{he2016deep}.
It started with an important family of deep network architectures called {\bf deep convolutional neural networks}
(DCNNs) which are very efficient for speech recognition, image classification, and many other practical tasks \cite{hinton2012deep, Krizhevsky2012}.
Compared with their great success in practice and some analysis of algorithms for training parameters like stochastic gradient descent,
DCNNs are not fully understood yet in terms of their approximation, modelling and generalization abilities.
Recently we confirm universality of DCNNs in \cite{zhou2020universality} and show in \cite{zhou2020theory}
that DCNNs can perform in representing functions at least as well as fully connected neural networks.
But it is open in general whether they can perform better in learning and approximating some classes of functions
with special structures used in practical applications, though there have been some attempts in \cite{mallat2016understanding, zhou2020theory, FFHZ2020}.

The first purpose of this paper is to answer the above open question by proving in Theorem \ref{lowerboundtheorem} below that
DCNNs followed by one fully connected layer can approximate {\bf radial functions} $f(|x|^2)$ much faster than fully connected shallow neural networks
where $|x| = \sqrt{x_1^2 + \ldots + x_d^2}$ is the norm of an input vector $x=(x_1, \ldots, x_d) \in\RR^d$.
In fact, we present dimension-independent rates of approximating radial functions in Theorem \ref{rateradial} below.
Moreover, we develop a theory of DCNNs for approximating efficiently functions of the form
$f \circ Q (x) =f(Q(x))$ with a polynomial $Q$ on $\RR^d$ and a univariate function $f$, both unknown.
Radial functions have such a form with a known quadratic polynomial $Q(x) = x_1^2 + \ldots + x_d^2$. They
arise naturally in statistical physics, early warning of earthquakes, 3-D point-cloud segmentation, and image rendering,
and their learning by fully connected neural networks was studied in \cite{McCane2018, ChuiLinZhou2019, ChuiLinZhou20192}.

The second purpose of this paper is to conduct {\bf generalization analysis} of a learning algorithm for regression induced by DCNNs
and to show for regression functions of the form $f \circ Q$ that the rates of estimation error (which equals the excess generalization error)
decrease to some optimal value and then increase as the depth of the deep network becomes large.
This is consistent with observations made in many practical applications of deep neural networks.

Our last purpose is to show that DCNNs in our network structure determined completely by two parameters
can automatically extract features and make use of the composite nature of the target function in learning for regression via tuning values of the two parameters, though our network structure is generic and does not use any composite information or the functions $Q$ and $f$. The activation function for our networks is the rectified linear unit (ReLU) $\sigma$ given by
$\sigma (u) = \max\{u, 0\}$ for $u\in\RR$.

A classical multi-layer fully connected neural network $\{h^{(j)}(x)\}_{j=0}^J$ of widths $\{d_j \in\NN\}$ takes an iterative form with $h^{(0)}(x)=x\in\RR^d$ and
$d_0 =d$ given by
\begin{equation}\label{fullconnectedN}
h^{(j)}(x)=\sigma\left(F^{(j)}h^{(j-1)}(x)-b^{(j)}\right), \qquad j=1,\dots,J,
\end{equation}
where $b^{(j)}\in\mathbb{R}^{d_j}$ and $F^{(j)}$ is a $d_j \times d_{j-1}$ full connection matrix reflecting the fully connected nature.
The number $d_j d_{j-1}$ of free parameters in the connection matrix $F^{(j)}$ is too large when the input dimension $d$ increases.
A core idea of deep learning is to reduce the number of free parameters at individual layers and channels by imposing special structures on the connection matrices.
The special structure imposed on DCNNs is induced by {\bf convolutions}. The 1-D convolution of a sequence $w=(w_k)_{k\in\ZZ}$ on $\ZZ$ supported in
$\{0, 1, \ldots, s\}$ and another $x=(x_k)_{k\in\ZZ}$ supported in $\{1, 2, \ldots, D\}$ is given by
$$ \left(w{*} x\right)_i = \sum_{k\in\ZZ} w_{i-k} x_k = \sum_{k=1}^D w_{i-k} x_k, \qquad i\in\ZZ. $$
This is a sequence supported in $\{1, 2, \ldots, D+s\}$. By restricting the index $i$ onto this set, we know that
the possibly nonzero entries of the convoluted sequence $w{*} x$ can be expressed in a vector form as
\begin{equation}\label{Toeplitz}
\left[\begin{array}{c}
\left(w{*} x\right)_1 \\
\left(w{*} x\right)_2 \\
\vdots \\
\left(w{*} x\right)_{D} \\
\vdots \\
\left(w{*} x\right)_{D+s}
\end{array}
\right]
=T^{w} \left[\begin{array}{c}
x_1 \\
x_2 \\
\vdots \\
x_{D}
\end{array}
\right], \quad
T^{w}:=\left[
 \begin{array}{ccccccc}
w_0 & 0 &0&0&\dots&0&0\\
w_1 &w_0 &0&0&\dots&0&0\\
\vdots&\vdots&\ddots&\ddots&\ddots&\vdots&\vdots\\
w_s &w_{s-1} &\dots&w_0 &\dots&0&0\\
0&w_s &\dots&w_1 &\ddots&\vdots&0\\
\vdots&\ddots&\ddots&\ddots&\ddots&\ddots&\vdots\\
\dots&\dots&0&w_s &\dots&w_1 &w_0 \\
\dots&\dots&\dots&0&w_s &\dots&w_1 \\
\vdots&\dots&\dots&\ddots&\ddots&\ddots&\vdots\\
0&\dots&\dots&\dots&\dots&0&w_s
\end{array}
\right].
\end{equation}
Here the Toeplitz type matrix $T^{w}$ is induced by the 1-D convolution and is called a {\bf convolutional matrix}.
The number of parameters $\{w_k\}_{k=0}^s$ contained in this structured connection matrix is $s+1$, much smaller than
the number of entries $D(D+s)$ of a full connection matrix of the same size. This great reduction at individual layers allows DCNNs to have large depths.
In this paper, we construct a deep neural network consisting of a group of $J_1 \in\ZZ_+$ comvolutional layers followed by a downsampling operation, and another group of $J_2 - J_1$ convolutional layers followed by a fully connected layer. The depth $J_2$ of the DCNNs and
the width of the last fully connected layer depend on an integer parameter $N \in\NN$ explicitly.
 For $u \geq 0$, we use $\lfloor u\rfloor$ to denote the integer part of $u$,
and $\lceil u\rceil$ the smallest integer greater than or equal to $u$.

\begin{definition}\label{dcnnconstr}
Let $x=(x_1, \ldots, x_d)\in\mathbb{R}^d$ be the input data vector, $s\in\NN$ be the filter length, and $J_1 \in\ZZ_+, N \in\NN$. The DCNN $\{h^{(j)}: \RR^d\to \RR^{d_j}\}_{j=1}^{J_2}$ with widths $\{d_j\}_{j=1}^{J_2}$ given by $d_0 =d$,
$d_{J_1} = \left\lfloor \frac{d+ J_1 s}{d} \right\rfloor$ and the iteration relation
$$ d_j = d_{j-1} + s, \qquad j\in\{1, \ldots, J_2\} \setminus \{J_1\} $$
has depth $J_2 := J_1 + \left\lceil \frac{(2N +3) d_{J_{1}}}{s-1}\right\rceil$ and is defined iteratively by $h^{(0)}(x)=x$ and
\begin{equation}\label{dcnn}
h^{(j)}(x)=\left\{\begin{array}{ll}
\sigma\left(T^{(j)}h^{(j-1)}(x)-b^{(j)}\right), & \hbox{if} \ j=\{1, \dots, J_2\} \setminus \{J_1\}, \\
\mathfrak{D}_d \left(\sigma\left(T^{(j)}h^{(j-1)}(x)-b^{(j)}\right)\right), & \hbox{if} \ j=J_1,
\end{array}\right.
\end{equation}
where $\{T^{(j)}:=T^{w^{(j)}}\}$ are the convolutional matrices induced by the sequence of filters ${\bf w} :=\{w^{(j)}\}_{j=1}^{J_2}$ each supported in $\{0,1,\dots,s\}$, $\mathfrak{D}_d: \RR^{d+ J_1 s} \to \RR^{\lfloor \frac{d+ J_1 s}{d} \rfloor}$ is the downsampling operator acting at the $J_1$-th layer given by
$$ \mathfrak{D}_d (v) = \left(v_{id}\right)_{i=1}^{\lfloor \frac{d+ J_1 s}{d}\rfloor}, \qquad v=\left(v_{i}\right)_{i=1}^{d +J_1 s} \in \RR^{d+ J_1 s}, $$
and $\{b^{(j)}\in\RR^{d_j}\}_{j=1}^{J_2}$ are bias vectors satisfying
\begin{equation}\label{biasrestr}
(b^{(j)})_{s+1}=(b^{(j)})_{s+2}=\ldots=(b^{(j)})_{d_{j-1}}, \qquad j=1, 2, \ldots, J_2 -1.
\end{equation}
The last layer $h^{(J_2 +1)}: \RR^d\to \RR^{2N +3}$ is produced with a connection matrix $F^{[J_2 +1]} \in \RR^{(2N +3) \times d_{J_2}}$ of identical rows and a bias vector $b^{(J_2 +1)} \in\RR^{2N +3}$ as
$$ h^{(J_2 +1)} (x)=\sigma\left(F^{[J_2 +1]}h^{(J_2)}(x)-b^{(J_2 +1)}\right). $$
The hypothesis space ${\mathcal H}_N$ for learning and approximation consists of all output functions depending on ${\bf w}$, $F^{[J_2 +1]}$, and the bias sequence ${\bf b} =\{b^{(j)}\}_{j=1}^{J_2 +1}$ as
\begin{equation}\label{hypothesisinfinity}
 {\mathcal H}_N =\left\{c\cdot h^{(J_2 +1)} (x): c\in \RR^{2N +3}, \ {\bf w}, \ {\bf b}, \ F^{[J_2 +1]}\right\}.
\end{equation}
\end{definition}

The restriction (\ref{biasrestr}) on the bias vectors $\{b^{(j)}\}_{j=1}^{J_2 -1}$ is  imposed based on the observation that the sums of the rows in the middle of the convolutional matrix $T^{w}$ in (\ref{Toeplitz})
are equal to $\sum_{k=0}^s w_k$.

If we introduce an activated affine map with a matrix $F$ and vector $b$ as
$$\mathcal{A}_{F,b}(v)=\sigma(Fv-b),\qquad v\in\mathbb{R}^{d_{j-1}},$$
then the last layer $h^{(J_2 +1)} (x)$ of our network can be expressed as
\begin{equation}
\mathcal{A}_{F^{[J_2 +1]}, b^{(J_2 +1)}} \circ \mathcal{A}_{T^{(J_2)}, b^{(J_2)}} \circ \ldots \circ \mathcal{A}_{T^{(J_1 +1)}, b^{(J_1 +1)}} \circ \mathfrak{D}_d \circ \mathcal{A}_{T^{(J_1)}, b^{(J_1)}}\circ \ldots \circ \mathcal{A}_{T^{(1)}, b^{(1)}}(x).
\end{equation}

The structure of the deep neural network in Definition \ref{dcnnconstr} is completely determined by the two parameters $J_1, N$ which are called {\it structural parameters}. This network structure does not involve any feature or composite information of the target functions. Once the structural parameters are chosen, the other parameters in (\ref{hypothesisinfinity}), ${\bf w}, {\bf b}, F^{[J_2 +1]}$ and $c$ can be trained with
stochastic configuration networks \cite{WangLi2017}, stochastic gradient descent, or some other randomized methods, and are called {\it training parameters}.

Traditional machine learning algorithms are often implemented in two steps of feature extraction and task-oriented
learning. In many practical applications, the first step of feature extraction is carried out with carefully designed preprocessing pipelines and data transformations and is labor intensive, involving feature engineering techniques,
human ingenuity and practical domain knowledge.

It has been believed from the great success of deep learning in practical applications that structures imposed on deep neural networks enable deep learning algorithms to combine automatically the two steps of extracting features and producing satisfactory outputs for desired learning tasks. We aim at verifying this belief for the convolutional structure imposed for CNNs in learning composite functions of the form $f \circ Q$. On one hand, the structure of our CNN network stated in Definition \ref{dcnnconstr} does not depend on the composite information of $f \circ Q$ or the functions $f, Q$; it is generic and determined only by two parameters $J_1, N$. On the other hand, if the target function takes the from $f \circ Q$, the convolutions enable our network to extract automatically the polynomial feature $Q$ and then learn the composite target function well, with tuned parameters $J_1, N$ of our unified DCNN model. We expect that our CNN network can extract some other nonlinear features and learn functions efficiently via tuning the two structural parameters.

To analyze the learning ability of the algorithm induced by our network, we use two novel ideas in estimating the approximation error and sample error. In our previous work \cite{zhou2020universality, zhou2020theory}, we have shown how to realize {\it linear features} $\{\xi_k \cdot x\}$ by a group of convolutional layers. In this paper we demonstrate how another group of convolutional layers together with a fully connected layer can be used to approximate ridge monomials $\{(\xi_k \cdot x)^\ell: 1\leq k \leq n_q, 1\leq \ell \leq q\}$ and then the polynomial $Q$. Here $n_q =\left(\begin{array}{c} d-1 + q \\ q \end{array}\right)$ is
the dimension of the space of homogeneous polynomials on $\RR^d$ of degree $q$, and $J_1 = \lceil \frac{n_q d-1}{s-1}\rceil$ is the number of convolutional layers in the first group.
Applying convolutional layers to extracting {\it nonlinear (polynomial) features} is the first novelty of this paper.
The second novelty is to bound the training parameters $c, {\bf w}, {\bf b}, F^{[J_2 +1]}$ in the expression (\ref{hypothesisinfinity}) of  the approximator in the approximation error part so that a bounded subset ${\mathcal H}_{R, N}$ of ${\mathcal H}_N$ (defined in (\ref{hypothesisRN}) below) contains the approximator and the covering numbers of the bounded hypothesis space ${\mathcal H}_{R, N}$ can be estimated for bounding the sample error.  This is achieved by applying Cauchy' bound of polynomial roots and Vieta's formula of
polynomial coefficients to bounding the filters constructed in convolutional factorizations of sequences.

\section{Main Results}

In this section we state our main results which will be proved in Sections \ref{construct}, \ref{mainproof}, and \ref{generalizationerror}. The approximation theorems given in the first two subsections show that if the target function has the composite form $f \circ Q$ with a Lipschitz-$\alpha$ function $f$, then the CNN network in Definition \ref{dcnnconstr}
achieves an approximation accuracy $\epsilon >0$ when the structural parameter $N$ is of order $O\left(\epsilon^{-\frac{1}{\alpha}}\right)$, a level for approximating univariate functions by neural networks. The learning rates stated in our last main result realized by the learning algorithm induced by our network for regression are of dimension-independent order $O\left(m^{-\frac{\alpha}{1+\alpha}}\right)$ for a sample of size $m$. These results tell us that the generic CNN network in Definition \ref{dcnnconstr} has the ability of automatically extracting the polynomial feature and making use of the composite nature of the target function via tuning the structural parameters $N, J_1$ in the learning process, even though the network does not involve any information about the polynomial or composition.

\subsection{Rates of approximating composite functions}

Without loss of generality we take the domain of definition of the approximated function
to be a subset of the unit ball $\Omega \subseteq {\mathbb B}:=\{x\in\RR^d: |x| \leq 1\}$. Denote $B_Q = \|Q\|_{C(\Omega)}$.
The univariate function $f$ in the composite form $f\circ Q$ of the approximated function is assumed to be in
$C^{0,\alpha}[-B_Q, B_Q]$ with $0<\alpha\leq1$, the space of Lipschitz-$\alpha$ functions on $[-B_Q, B_Q]$ with semi-norm $|f|_{C^{0,\alpha}}:=
\sup\limits_{x_1 \not= x_2\in [-B_Q, B_Q]}\frac{|f(x_1)-f(x_2)|}{|x_1-x_2|^\alpha}$.

To achieve desired estimates for the approximation error to be used in our generalization analysis, we shall construct an approximator from ${\mathcal H}_N$ and choose some training parameters explicitly in terms of the degree of the feature polynomial $Q$: the filters $\{w^{(j)}\}_{j=J_1 +1}^{J_2}$ by
(\ref{Wsequencesecond}) and (\ref{filterW1}), connection matrix $F^{[J_2 +1]}$ and bias vector $b^{(J_2 +1)}$ by (\ref{FNNFb}) in Section  \ref{construct}. This reduces the total number of implicit training parameters which are called {\it free parameters} in our approximator construction, to distinguish them from the training parameters in Definition \ref{dcnnconstr} and our generalization analysis.

Our first main result provides rates of approximating $f\circ Q$ by DCNNs with a downsampling operation followed by a fully connected layer.

\begin{theorem}\label{ratefQ}
Let $2 \leq s \leq d$, $Q$ be a polynomial on $\Omega$ of degree at most $q \in\NN$, and $f\in C^{0,\alpha}[-B_Q, B_Q]$ for some $0< \alpha \leq 1$. Take $J_1 = \lceil \frac{n_q d-1}{s-1}\rceil$.
Then for any $N\in\NN$, there exists a deep network stated in Definition \ref{dcnnconstr}
with $\{w^{(j)}\}_{j=J_1 +1}^{J_2}$, $F^{[J_2 +1]}$ and $b^{(J_2 +1)}$ explicitly constructed (given in Sections \ref{construct})
such that
\begin{equation}\label{ratefQexplicit}
\inf \left\{\left\|c \cdot h^{(J_2 +1)}
- f \circ Q\right\|_{C(\Omega)}: \|c\|_\infty \leq \frac{4\|f\|_\infty N}{\widehat{B}_Q}\right\}  \leq \frac{C_{Q, \alpha} |f|_{C^{0, \alpha}}}{N^\alpha},
\end{equation}
where $\widehat{B}_Q$ is a constant depending only on $Q$ and $C_{Q, \alpha}$ is another one only on $Q, \alpha$.

The total number ${\mathcal N}$ of free parameters in this network can be bounded as
$$ {\mathcal N} \leq \left(30 + 28 n_q\right) N + (8 d +q + 44) n_q + 4 s + 49. $$

To achieve the approximation accuracy $0< \epsilon \leq 1$, we only need
\begin{equation}\label{complexityfQ}
{\mathcal N} \leq \left(C_{Q, \alpha, |f|_{C^{0, \alpha}}} + d\right) d^{q} \epsilon^{-\frac{1}{\alpha}}
\end{equation}
free parameters by taking
$N = \left\lceil C_{Q, \alpha}^{\frac{1}{\alpha}} |f|_{C^{0, \alpha}}^{\frac{1}{\alpha}} \epsilon^{-\frac{1}{\alpha}}\right\rceil,$
where $C_{Q, \alpha, |f|_{C^{0, \alpha}}}$ is a constant depending on $Q, \alpha, |f|_{C^{0, \alpha}}$, but not on $d$ or $\epsilon$.
\end{theorem}

Theorem \ref{ratefQ} tells us that the total number of free parameters in the constructed network
for achieving the approximation accuracy $\epsilon$ is $O\left(\epsilon^{-\frac{1}{\alpha}}\right)$,
which has the same complexity as that of a fully connected shallow network (\ref{fullconnectedN}) with $J=1$
for approximating a univariate Lipschitz-$\alpha$ function. But the constant term in the bound (\ref{complexityfQ})
demonstrates the role of the polynomial feature $Q$ and the data dimension $d$ in the network complexity.

\subsection{Rates of approximating radial functions}

When the feature polynomial is the special one $Q(x) = |x|^2$ for learning radial functions, we only need the second group of convolutional layers without downsampling by taking the parameter value $J_1 =0$.

If we apply Theorem \ref{ratefQ} to a radial function $f(|x|^2)$ with $Q(x) =|x|^2$, we see that
the constant in the complexity bound (\ref{complexityfQ}) is at least of order $O\left(d^{3}\right)$ with respect to the data dimension $d$.
The constructed network for achieving the complexity in Theorem \ref{ratefQ} uses the degree of the feature polynomial $Q$, not
the exact form of $Q$. However, the quadratic polynomial $Q$ for a radial function $f(|x|^2)$ is known. Making use of this exact form of $Q$
enables us to reduce the depth of the network for approximating radial functions and improve the constant in the complexity bound to an order $O\left(d^{2}\right)$.

\begin{theorem}\label{rateradial}
Let $2 \leq s \leq d$ and $f\in C^{0,\alpha}[0, 1]$ for some $0< \alpha \leq 1$.
Then for any $N\in\NN$, there exists a deep network stated in Definition \ref{dcnnconstr}
with $J_1 =0, J_2 = J := \left\lceil \frac{(2N +3) d}{s-1}\right\rceil$ and $\{w^{(j)}\}_{j=1}^{J}$, $F^{[J +1]}$ and $b^{(J +1)}$ explicitly constructed (given in Sections \ref{construct})
such that the last layer
$$ h^{(J +1)} (x)= \mathcal{A}_{F^{[J +1]}, b^{(J +1)}} \circ \mathcal{A}_{T^{(J)}, b^{(J)}} \circ \ldots \circ \mathcal{A}_{T^{(1)}, b^{(1)}}(x) $$
satisfies
\begin{equation}\label{rateradialbound}
\inf \left\{\left\|c \cdot h^{(J +1)}(x)
- f (|x|^2)\right\|_{C(\Omega)}: \left\|c\right\|_\infty \leq \frac{4N \|f\|_\infty}{1+4d}\right\}  \leq 3 \left(1+ 4d\right)^\alpha |f|_{C^{0, \alpha}} N^{-\alpha}.
\end{equation}
The total number ${\mathcal N}$ of free parameters in this network can be bounded as
$$ {\mathcal N} \leq (14 d +2) N + 22 d + s +1. $$
To achieve the accuracy $0< \epsilon \leq 1$, by taking $N = \left\lceil \left(1+ 4d\right) \left(3 |f|_{C^{0, \alpha}}\right)^{\frac{1}{\alpha}} \epsilon^{-\frac{1}{\alpha}}\right\rceil$ we know that the total number of free parameters can be bounded as
$$ {\mathcal N} \leq  15 \left( 5 \left(3 |f|_{C^{0, \alpha}}\right)^{\frac{1}{\alpha}} +1\right) d^2 \epsilon^{-\frac{1}{\alpha}} + 22 d + s +1. $$
\end{theorem}

\subsection{Super efficiency in approximating radial functions}
\label{lowerboundofapproximation}

Our third main result demonstrates that deep neural networks have super efficiency in approximating radial functions, compared with shallow networks.
Consider the set of radial functions in the unit ball of the space $C^{0, 1}({\mathbb B})$ of Lipschitz functions on ${\mathbb B}$ defined by
\begin{equation}\label{radialLip1}
{\mathcal B}\left(C^{0, 1}_{|\cdot|}\right):= \left\{f(|\cdot|^2): \|f(|\cdot|^2)\|_{C^{0, 1}({\mathbb B})} \leq 1\right\},
\end{equation}
where $\|f(|\cdot|^2)\|_{C^{0, 1}({\mathbb B})}$ is the Lipschitz-$1$ norm of the function $f(|\cdot|^2)$ defined for functions $g$ on ${\mathbb B}$ by
\begin{equation}\label{radialLip1normde}
\|g\|_{C^{0, 1}({\mathbb B})} = \sup_{x \not= y \in {\mathbb B}} \frac{\left|g(x)- g(y)\right|}{|x-y|} + \sup_{x\in {\mathbb B}} \left|g(x)\right|.
\end{equation}

Denote the span of $N$ ridge functions as
\begin{equation}\label{radialLip1normde}
{\mathcal S}_N =\left\{\sum_{k=1}^N c_k \sigma_k (a_k \cdot x -b_k): \ \sigma_k  \in C(\RR), \ a_k \in\RR^d, \ c_k, b_k \in\RR\right\}.
\end{equation}
Recall the hypothesis space generated by a shallow neural network is a subset of ${\mathcal S}_N$ consisting of functions with $\sigma_1 =\ldots =\sigma_N$ being an activation function.

The efficiency of a neural network generating a hypothesis space $V$ in approximating a set $U$ of functions on $\Omega ={\mathbb B}$ uniformly is measured by the quantity
\begin{equation}\label{distdef}
\hbox{dist}(U, V):= \sup_{f \in U} \inf_{g \in V} \|f-g\|_{L_\infty({\mathbb{B}})}
\end{equation}
which is the deviations of $U$ from $V$ in $L_\infty({\mathbb{B}})$.

\begin{theorem} \label{lowerboundtheorem}
Let $2\leq s \leq d$. We have
\begin{equation}\label{shallowbound}
\hbox{dist}\left({\mathcal B}\left(C^{0, 1}_{|\cdot|}\right), {\mathcal S}_N\right) \geq c_d N^{-\frac{1}{d-1}}, \qquad \forall N\in\NN
\end{equation}
with a constant $c_d$ independent of $N$; while for the hypothesis space ${\mathcal H}_N$ generated by the deep network constructed in
Theorem \ref{rateradial}, there holds
\begin{equation}\label{deepbound}
\hbox{dist}\left({\mathcal B}\left(C^{0, 1}_{|\cdot|}\right), {\mathcal H}_N\right) \leq 3 \sqrt{1+ 4d} N^{-\frac{1}{2}}, \qquad \forall N\in\NN.
\end{equation}
\end{theorem}

By Theorem \ref{lowerboundtheorem}, we know that the total number of free parameters in our DCNN network
for achieving an accuracy $\epsilon >0$ in approximating functions from the class ${\mathcal B}\left(C^{0, 1}_{|\cdot|}\right)$
is $O(\epsilon^{-2})$ while that of a fully connected shallow network is $O(\epsilon^{-(d-1)})$.
This shows that deep neural networks are much more efficient than shallow networks in approximating radial functions when the dimension $d >3$ is large.

\subsection{Generalization analysis of DCNNs}\label{generalizationerror}

Our last main result is analysis of the generalization ability of the
commonly used empirical risk minimization (ERM) algorithm over a bounded hypothesis space ${\mathcal H}_{R, N}$ generated by our deep neural network. Here $R>0$ is used for bounding the parameters of the output functions from the hypothesis space ${\mathcal H}_{N}$ as
\begin{eqnarray}
{\mathcal H}_{R, N}= &&\Big\{c \cdot h^{(J +1)}(x): \ \|w^{(j)}\|_\infty \leq R,  \|b^{(j)}\|_\infty \leq (2(s+1)R)^j, \ \forall j, \ \|c\|_\infty \leq N R, \nonumber \\
&& F^{[J_2+1]} \ \hbox{has identical rows with norm} \ \|F^{[J_2+1]}\|_\infty \leq N^2 R\Big\}, \label{hypothesisRN}
\end{eqnarray}
where $\|F^{[J_2+1]}\|_\infty$ is the norm of the matrix $F^{[J_2+1]}$ as a linear operator from $\left(\RR^{d_{J_2}}, \|\cdot\|_\infty\right)$ to $\left(\RR^{2N+3}, \|\cdot\|_\infty\right)$ which equals the maximum of the $\ell_1$-norms of its rows.

We follow the classical learning framework for regression which can be found in \cite{CuckerZhou2007}.
A data sample $D=\{(x_i,y_i)\}^{m}_{i=1} \subset \mathcal{Z}^m$ is independently drawn from a Borel probability measure $\rho$
on $\mathcal{Z}=\mathcal{X}\times\mathcal{Y}$ with $\mathcal{X}=\Omega$ and $\mathcal{Y}\subseteq[-M,M]$ for some $M>0$. The target function for learning is the regression function $f_\rho$ on $\mathcal{X}$ defined by
$f_\rho(x)=\int_{\mathcal{Y}} yd\rho(y|x)$ minimizing the generalization error $\mathcal{E}(f):=\int_{\mathcal{Z}} (f(x)-y)^2d\rho$,
where $\rho(y|x)$ denotes the conditional distribution at $x\in\mathcal{X}$ induced by $\rho$.
Denote by $\rho_X$ the marginal distribution of $\rho$ on $\mathcal{X}$ and by $(L^2_{\rho_X},\parallel\cdot\parallel_\rho)$
the Hilbert space of square integrable functions with respect to $\rho_X$.

The ERM algorithm defined on the bounded hypothesis space (\ref{hypothesisRN}) learns an empirical target function as
\begin{equation}\label{ERMalgorithm}
f_{D, R, N}:=\arg \min_{f\in {\mathcal H}_{R, N}} \dfrac{1}{m} \sum_{i=1}^{m} (f(x_i)-y_i)^2.
\end{equation}

Let us emphasize again that the structure of the network for defining the bounded hypothesis space in the above algorithm is completely determined by the two parameters $J_1$ and $N$, and it does use any property of the regression function. However, if the regression function takes a composite form $f_\rho  = f \circ Q$, the learning algorithm (\ref{ERMalgorithm}) has the ability of automatically extracting features and making use of the composite property via tuning $J_1, N$ so that
the empirical target function $f_{D, R, N}$ can learn $f_\rho$ in the same learning rates as for learning univariate functions.

Since $\mathcal{Y}\subseteq[-M,M]$, we project the output function onto the interval $[-M, M]$ and introduce the truncated empirical target function
$$\pi_{M} f_{D, R, N}(x):=\left\{\begin{array}{ll}
f_{D, R, N}(x), & \hbox{if} \ |f_{D, R, N}(x)| \leq M, \\
M, & \hbox{if} \ f_{D, R, N}(x)>M, \\
- M, & \hbox{if} \ f_{D, R, N}(x) <-M.
\end{array}\right. $$

\begin{theorem} \label{generalizationerrortheorem}
Let $2 \leq s \leq d$ and $Q$ be a polynomial on $\Omega$ of degree at most $q\in\NN$. If $f_\rho  = f \circ Q$ for some
$f\in C^{0,\alpha}[-B_Q, B_Q]$ with some $0< \alpha \leq 1$, then for $N\in\NN$ and $R \geq R_{q, d, s, Q, \|f\|_\infty}$, we have
\begin{equation}\label{ratesbound}
\mathbb{E}\left[\|\pi_M f_{D, R, N}-f_\rho\|_\rho^2\right] \leq C_{Q, s, d, \alpha, M, |f|_{C^{0, \alpha}}} \log \left(2(s+1)R\right) \max\left\{N^{-2\alpha}, \ \frac{N^2}{m}\right\},
\end{equation}
where $C_{Q, s, d, \alpha, M, |f|_{C^{0, \alpha}}}$ is a constant independent of $m, N$ or $R$, and $R_{q, d, s, Q, \|f\|_\infty}$
is a constant depending on $q, s, d, Q, \|f\|_\infty$ (given explicitly in Lemma \ref{boundweightslemma} below). In particular, by choosing
$N =\left\lceil m^{\frac{1}{2+ 2 \alpha}}\right\rceil$, we have
$$
\mathbb{E}\left[\|\pi_M f_{D, R, N}-f_\rho\|_\rho^2\right] \leq 4 C_{Q, s, d, \alpha, M, |f|_{C^{0, \alpha}}} \log \left(2(s+1)R\right) m^{-\frac{\alpha}{1+ \alpha}}.
$$
\end{theorem}

Observe from Theorem \ref{ratefQ} that the approximation error bound (\ref{ratefQexplicit}) always decreases as the depth of the network $J_2 +1 \approx N$ increases.
On the other hand, the capacity of the hypothesis space ${\mathcal H}_{R, N}$ increases with the depth, which leads to larger sample error.
Combining these two terms of error, we see from Theorem \ref{generalizationerrortheorem} that the estimation error bound (\ref{ratesbound})
decreases to a minimum as the network depth increases to an optimal value, and then increases as the depth becomes larger.
This trade-off phenomenon is common in many applications of deep learning algorithms.

\section{Comparisons and Discussion}

Establishing a solid theoretical foundation for deep learning is greatly desired. A core challenge is to prove
that structured deep neural networks used in deep learning can outperform the classical fully connected networks
and automatically extract features when the data or target functions take forms involving some special features. The main difficulty lies in
the approximation theory of structured deep neural networks like DCNNs which is totally different from the nice theory
for fully connected networks developed about $30$ years ago. For example,
a typical approximation rate in \cite{Mhaskar1993} obtained by a localized Taylor expansion approach
for shallow network (\ref{fullconnectedN}) of width $N$ with $J=1$ is
$\inf_{c} \left\|c \cdot h^{(1)} - f\right\|_{C([-1, 1]^d)} =O(N^{-\alpha/d})$ for $f \in W_\infty^\alpha ([-1, 1]^d)$,
when the activation function $\sigma$ is $C^\infty$ sigmoid type satisfying for some $b\in{\mathbb R}$ and some integer $\ell \in\NN\setminus\{1\}$,
a restriction $\sigma^{(k)} (b)\not= 0$ for all $k\in\ZZ_+$ and an asymptotic condition
$\lim_{u\to \infty} \sigma(u)/u^\ell=1$. Such approximation rates were recently proved for ReLU shallow networks
in \cite{Klusowski2018} for functions satisfying a decay condition for the Fourier transform \cite{Barron}, and were extended to ReLU deep networks
in \cite{Safran2017, Telgarsky 2016, Yarosky, Grohs, Petersen, Nakada2019, Sonoda, Suzuki2019} for functions from $W_\infty^\alpha ([-1, 1]^d)$ with $0< \alpha \leq 2$,
and in \cite{Shaham} for approximation on manifolds, all for fully connected networks.
In particular, it was shown in \cite{Safran2017, Yarosky} that an accuracy $\epsilon>0$ for approximating functions from $W_\infty^\alpha ([-1, 1]^d)$ can be achieved
by a deep ReLU fully connected network of depth $\frac{C_0 d}{4} (\log (1/\epsilon) + d)$ with $C_0 >0$ and
$2^d \epsilon^{-d/\alpha}$ free parameters which has the same order of complexity as required by shallow sigmoid networks.

Deep CNNs use convolutional matrices (\ref{Toeplitz}) which have a special structure
with sparsity, making their approximation theory totally different. We show
in \cite{zhou2020universality} that an accuracy $\epsilon>0$ for approximating a function
from $W^\alpha_2\left([-1,1]^d\right)$ with an integer index $\alpha>2+d/2$
can be achieved by a DCNN of depth $4 \lceil \frac{1}{\epsilon^2} \log  \frac{1}{\epsilon^2}\rceil$ and a linearly increasing number of
$\lceil \frac{75}{\epsilon^2} \log  \frac{1}{\epsilon^2}\rceil d$ free parameters which improves
the bound in Theorem 1 of \cite{Yarosky} with respect to the data dimension $d$.
The restriction $\alpha>2+d/2$ on the smoothness was relaxed by a spherical analysis approach
in our recent work \cite{FFHZ2020} when the data are from the unit sphere of $\RR^d$.
We also show in \cite{zhou2020theory} that DCNNs can realize the output layer of any fully-connected neural network with the same order of complexity.
This observation was made for periodized DCNNs with different architectures and connection matrices
in \cite{Oono2019, Petersen2020}.

All the above estimates on approximation by deep neural networks, structured or fully connected,
are stated in terms of the smoothness of the approximated function. Approximating radial functions
by fully-connected neural networks was studied in \cite{McCane2018, ChuiLinZhou2019, ChuiLinZhou20192},
while representing functions with variables having given compositional structures by fully-connected networks designed based on the known compositional structures was considered in \cite{MhaskarPoggio, Poggio2020}.

In this paper we present estimates for learning compositional functions $f\circ Q$ with polynomial features $Q$ by an ERM algorithm with hypothesis spaces generated by DCNNs without involving any composite property or functions $f, Q$. We verify that our generic DCNN network can automatically make use of the composite nature of the target function and the polynomial feature $Q$ by tuning two parameters $J_1, N$.
We show rigorously that DCNNs can outperform fully connected shallow networks
in approximating the class of radial functions when the data dimension $d>3$ is large. To our best knowledge,
this is the first proof for the superiority of DCNNs in approximating functions with structures,
though some hints were provided in our previous work \cite{zhou2020theory, FFHZ2020} on approximating ridge functions or
additive ridge functions of the form $\sum_{k=1}^K g_k (\xi_k \cdot x)$ with $\xi_k \in\RR^d$.
It would be interesting to apply our ideas to some other learning problems \cite{FanHuWuZhou, GXGZ, LinZhou}, and to
investigate more function structures and features for which DCNNs combined with pooling, channels, and other network architectures \cite{ZhouZhou}
can demonstrate super efficiency in feature extraction, approximations and representations of multivariate functions.

Generalization analysis for ERM with fully connected neural networks has been well developed in the literature \cite{Kohler2017, Schmidt-Hieber, ChuiLinZhou2019}. 
For composite functions, deep fully connected ReLU networks were constructed in \cite{Schmidt-Hieber} based on composite dimensions 
and some ideas from \cite{Yarosky} to achieve optimal learning rates. In particular, it was shown by (26) there 
that a network output function $f^{*}$ yields $\left\|f^{*}- f \circ Q\right\|_{C(\Omega)} =O\left({\mathcal N}^{-\alpha}\right)$ 
in approximating $f\circ Q$ with $f\in C^{0,\alpha}$. This rate of approximation by deep fully connected networks is the same as (\ref{ratefQexplicit}) in 
our Theorem \ref{ratefQ} for approximation by deep CNNs, while the dimension dependence of the constant in our estimate is better, as seen more explicitly 
in (\ref{rateradialbound}) of Theorem \ref{rateradial} for approximating radial functions. 
Together with sample error estimates, our approximation theory gives generalization analysis with error bound (\ref{ratesbound}) for ERM with DCNNs followed by a fully connected layer.
The bound decreases with the depth to a minimum and then increases, which verifies a trade-off phenomenon observed in practice.

\section{Constructing Deep Network for Approximation}\label{construct}

In this section, we construct a deep neural network with downsampled deep convolutional layers followed by a fully connected layer
for approximating a composite function $f\circ Q$ which will be used to prove our first two main results.
Our construction makes full use of the special structure of the approximated function induced by the polynomial $Q$ and the univariate function $f$.
It is based on an important fact \cite{LinPinkus1993, Maiorov1999a} on the space ${\mathcal P}_q^h (\RR^d)$ of homogeneous polynomials on $\RR^d$ of degree $q$
that ${\mathcal P}_q^h (\RR^d)$ has a basis $\{(\xi_k \cdot x)^q\}_{k=1}^{n_q}$ for some vector set $\{\xi_k\}_{k=1}^{n_q} \subset \RR^d \setminus \{0\}$
and this vector set can even be chosen in such way that the homogeneous polynomial set $\{(\xi_k \cdot x)^\ell\}_{k=1}^{n_q}$ spans the space ${\mathcal P}_\ell^h (\RR^d)$ for every $\ell \in \{1, \ldots, q-1\}$.
Applying this fact to the polynomial $Q$ of degree $q$ yields the following lemma stated in \cite{zhou2018deepdistributed}.
Take $J_1 = \lceil \frac{n_q d-1}{s-1}\rceil$ in this section except in the proof of Theorem \ref{rateradial}.

\begin{lemma}\label{ridgepolynomialiden}
Let $d \in\NN$ and $q \in\NN$. Then there exists a set $\{\xi_k\}_{k=1}^{n_q} \subset \{\xi \in\RR^d: |\xi| =1\}$ of vectors with $\ell_2$-norm $1$ such that
for any $Q \in {\mathcal P}_q (\RR^d)$ we can find a set of coefficients $\{\beta_{k, \ell}: k=1, \ldots, n_q, \ell=1, \ldots, q\} \subset \RR$ such that
\begin{equation}\label{ridgepolynomial}
Q  (x) =Q(0) + \sum_{k =1}^{n_q} \sum_{\ell=1}^q \beta_{k, \ell} (\xi_k \cdot x)^\ell, \qquad x\in \RR^d.
\end{equation}
\end{lemma}

Our deep neural network consists of two groups of CNN layers and one fully connected layer:
the first group with a downsampling operation for realizing the linear features $\{\xi_k \cdot x\}_k$,
the second for produce ridge functions $\{\sigma(\xi_k \cdot x -t_j)\}_{k, j}$ leading to realizing $\{(\xi_k \cdot x)^\ell\}_{k, \ell}$ in (\ref{ridgepolynomial}),
and the last fully connected layer for approximating the univariate function $f$ to achieve an approximation of $f(Q(x))$.

\subsection{Realizing linear features by DCNNs}\label{firstCNN}

To realize the linear features $\{\xi_k \cdot x\}_{k=1}^{n_q}$, we apply the following lemma from \cite{zhou2018deepdistributed, zhou2020theory}.

\begin{lemma}\label{convolutionconstrct}
Let $s \geq 2$ and $W=(W_k)_{k=-\infty}^{\infty}$ be a sequence supported in $\{0, \ldots, {\mathcal M}\}$ with ${\mathcal M} \geq 0$.
Then there exists a finite sequence of filters $\{w^{(j)}\}_{j=1}^{p}$ each supported in $\{0, \ldots, s\}$ with $p \leq \lceil \frac{{\mathcal M}}{s-1}\rceil$ such that the following convolutional factorization holds true
\begin{equation}\label{filterI}
W  = w^{(p)}{*}w^{(p-1)}{*}\ldots {*}w^{(2)}{*}w^{(1)}.
\end{equation}
\end{lemma}

We show that the components of $h^{(J_{1})} (x)$ can realize the linear features $\{\xi_k \cdot x\}_{k=1}^{n_q}$.

\begin{lemma}\label{initiallayers}
Let $\{\xi_k\}_{k=1}^{n_q} \subset \{\xi \in\RR^d: |\xi| =1\}$. There
exist filters $\{w^{(j)}\}_{j=1}^{J_1}$ each supported in $\{0, \ldots, s\}$ with $J_1 = \lceil \frac{n_q d-1}{s-1}\rceil$
and bias vectors $\{b^{(j)}\}_{j=1}^{J_1}$ each satisfying the restriction in (\ref{biasrestr}) such that the $J_1$-th layer after downsampling is
\begin{equation}\label{initialexplicit}
h^{(J_{1})} (x) = \left[\begin{array}{c}
\xi_1 \cdot x \\
\vdots \\
\xi_{n_q} \cdot x \\
\vdots \\
\xi_{d_{J_1}} \cdot x
\end{array}\right] + B {\bf 1}_{d_{J_{1}}},
\end{equation}
where $\{\xi_k\}_{k=n_q +1}^{d_{J_1}} \subset \RR^d$, $B$ is a positive constant, and ${\bf 1}_{d_{J_{1}}}$ the constant $1$ vector in $\RR^{d_{J_{1}}}$. Moreover,
we have
\begin{equation}\label{positiveness}
\left|\xi_k \cdot x\right| \leq  B, \qquad \forall x\in \Omega, k=1, \ldots, d_{J_{1}}.
\end{equation}
The number ${\mathcal N}_1$ of free parameters in the first $J_1$ layers is
$$ {\mathcal N}_1 = (3s+2) J_1 = (3s+2) \left\lceil \frac{n_q d-1}{s-1}\right\rceil. $$
\end{lemma}

\begin{proof} We use some ideas and results from \cite{zhou2020theory} to prove our conclusion.

First, we define a sequence $W$ supported in $\{0, 1, \ldots, n_q d-1\}$ by stacking the vectors $\{\xi_k\}_{k=1}^{n_q}$ (with the components of each vector reversed) as
$$ W_{j + (k-1) d} = \left(\xi_k\right)_{d-j}, \qquad j=0, 1, \ldots, d-1, \ k=1, 2, \ldots, n_q. $$
Then
\begin{equation}\label{Wk}
\left[W_{d-1 + (k-1) d} \ \ldots \ W_{1 + (k-1) d} \ W_{(k-1) d}\right] = \xi_k^T, \qquad k=1, 2, \ldots, n_q.
\end{equation}
By Lemma \ref{convolutionconstrct}, there exist filters $\{w^{(j)}\}_{j=1}^{p}$
each supported in $\{0, \ldots, s\}$ with $p \leq \lceil \frac{n_q d-1}{s-1}\rceil$ such that (\ref{filterI}) holds true.
Taking $\{w^{(j)}\}_{j=p+1}^{J_1}$ to be the delta sequence $\delta_0$ given by $\left(\delta_0\right)_i =0$ for $i\in\ZZ \setminus \{0\}$
and $\left(\delta_0\right)_0 =1$ yields
\begin{equation}\label{firstlayerfactor}
W  = w^{(J_1)}{*}w^{(J_1-1)}{*}\ldots {*}w^{(2)}{*}w^{(1)}.
\end{equation}
Hence we know from \cite[Lemma 1]{zhou2020theory} that the filters $\{w^{(j)}\}_{j=1}^{J_1}$ induce Toeplitz type matrices $\{T^{(j)}\in\RR^{(d_{j-1}+s) \times d_{j-1}}\}_{j=1}^{J_1}$ which satisfy the matrix product identity
\begin{equation}\label{filterJmatrix}
T^{(J_1)} T^{(J_{1}-1)} \ldots T^{(2)} T^{(1)} = \left(W_{i-k}\right)_{i=1, \ldots, d + J_{1}s, \ k=1, \ldots, d}.
\end{equation}
Here the matrix on the right-hand side takes the form
\begin{equation}\label{matrixTW}
\left[\begin{array}{llll}
W_0 & 0 & \cdots  & 0 \\
W_1 & W_0 & \ddots  & 0  \\
\vdots & \ddots & \ddots & \vdots \\
W_{d-1} & \cdots &  W_1 & W_0 \\
W_{d} & W_{d-1} & \cdots & W_1 \\
\vdots & \ddots &  \ddots & \vdots  \\
W_{J_1 s} & \cdots & \ddots & \vdots \\
0 & W_{J_1 s} & \cdots & \vdots \\
\vdots & \ddots & \ddots & \vdots  \\
0 & \cdots & 0 & W_{J_1 s}
\end{array}\right].
\end{equation}
Note that $d+ J_1 s = d + \lceil \frac{n_q d-1}{s-1}\rceil s \geq d + \frac{n_q d-1}{s-1} s \geq d + n_q d$ which implies
$$d_{J_1} =\left\lfloor \frac{d+ J_1 s}{d} \right\rfloor \geq n_q +1. $$

Then we choose the bias vectors $\{b^{(j)}\}_{j=1}^{J_1}$ by $b^{(1)} =  - \|w^{(1)}\|_1 {\bf 1}_{d_{1}}$ and
$$ b^{(j)} = \left(\Pi_{p=1}^{j-1} \|w^{(p)}\|_1\right)  T^{(j)} {\bf 1}_{d_{j-1}} - \left(\Pi_{p=1}^{j} \|w^{(p)}\|_1\right) {\bf 1}_{d_{j-1} + s}, \qquad j=2, \ldots, J_1, $$
where $\|w\|_1=\sum_{k\in {\mathbb Z}} |w_k|$ denotes the $\ell^1$-norm of a finitely supported sequence $w$. With this choice we know
from \cite[Lemma 3]{zhou2020theory} that
$$ h^{(J_{1})} (x) = {\mathfrak D}_d \left(T^{(J_1)}  T^{(J_{1}-1)} \ldots T^{(2)} T^{(1)} x + \left(\Pi_{p=1}^{J_{1}} \|w^{(p)}\|_1\right) {\bf 1}_{d_{J_{1}}}\right). $$
Combining this with (\ref{Wk}) and (\ref{filterJmatrix}), we see that the $k$-th component of $h^{(J_{1})} (x)$ equals
$$ \left[W_{d-1 + (k-1) d} \ \ldots \ W_{1 + (k-1) d} \ W_{(k-1) d}\right] x + \Pi_{p=1}^{J_{1}} \|w^{(p)}\|_1 =
\xi_k \cdot x + \Pi_{p=1}^{J_{1}} \|w^{(p)}\|_1, $$
where the vectors $\xi_k$ with $k=1, \ldots, n_q$ coincide with those from the given set $\{\xi_k\}_{k=n_q +1}^{d_{J_1}} \subset \RR^d$.
This verifies the desired expression (\ref{initialexplicit}) with the constant $B= \Pi_{p=1}^{J_{1}} \|w^{(p)}\|_1$ and
$\xi_k = \left[W_{kd-1} \ W_{kd-2} \ \ldots W_{kd-d}\right]^T \in\RR^d$ for $k>n_q.$ We also see that (\ref{positiveness}) holds true.

Observe that each filter $w^{(j)}$ has $s+1$ free parameters to be determined and each bias vector $b^{(j)}$ has
$2s+1$ free parameters. Then the stated expression for ${\mathcal N}_1$ is valid.
This proves the lemma.
\end{proof}

\subsection{Producing ridge functions by DCNNs}\label{secondCNN}

The second group of convolutional layers in our network is used to produce ridge functions $\{\sigma(\xi_k \cdot x -t_j): k=1, \ldots, n_q, j=1, \ldots, 2N +3\}$
with $N\in\NN$ where
\begin{equation}\label{knottj}
t_j=-1+\frac{j-2}{N}, \qquad j=1, \ldots, 2N+3.
\end{equation}
This is done by conducting a convolutional factorization of a sequence $W^{[1]}$ supported on $\{0, \ldots, (2N +3) d_{J_{1}}\}$ given by
\begin{equation}\label{Wsequencesecond}
W^{[1]}_{i} = \left\{\begin{array}{ll}
1, & \hbox{if} \ i\in\{kd_{J_{1}}\}_{k=0}^{2N +3}, \\
0, & \hbox{otherwise.}\end{array}\right.
\end{equation}
If we denote the symbol $\tilde{u}$ of a filter $u$ supported in $\ZZ_+$ to be a polynomial on $\CC$ given by $\tilde{u}(z) = \sum_{k\in\ZZ} u_k z^k$,
then the symbol $\widetilde{W^{[1]}}$ of the sequence $W^{[1]}$ is given by
$$ \widetilde{W^{[1]}} (z) = \sum_{k=0}^{2N +3} z^{k d_{J_{1}}}, \qquad z\in \CC. $$
It has $(2N +3) d_{J_{1}}$ complex roots
$$ e^{\frac{i 2 \ell \pi}{d_{J_{1}} (2N +4)}}, \qquad 1 \leq \ell \leq d_{J_{1}} (2N +4) -1, \ (2N +4)\not|\ell. $$
Applying a procedure for convolutional factorization stated in \cite{zhou2020theory},
we can find explicit expressions without free parameters involved for the filters $\{w^{(j)}\}_{j=J_1 +1}^{J_2}$, each supported in $\{0, \ldots, s\}$, with
$J_2 = J_1 + \left\lceil \frac{(2N +3) d_{J_{1}}}{s-1}\right\rceil$ such that
\begin{equation}\label{filterW1}
W^{[1]}  = w^{(J_2)}{*}\ldots {*}w^{(J_1 +2)}{*}w^{(J_1 +1)}.
\end{equation}
Then the second group of convolutional layers is constructed as follows.

\begin{lemma}\label{DCNNapprox}
For $h^{(J_{1})} (x)$ given by (\ref{initialexplicit}), $N \in\NN$, and the filters $\{w^{(j)}\}_{j=J_1 +1}^{J_2}$
explicitly constructed above satisfying (\ref{Wsequencesecond}) and (\ref{filterW1}), there exist
bias vectors $\{b^{(j)}\}_{j=J_1 + 1}^{J_2}$ satisfying (\ref{biasrestr}) such that
\begin{equation}\label{secondlayerexplicit}
\left(h^{(J_{2})} (x)\right)_{(j-1)d_{J_{1}} + k} = \left\{\begin{array}{ll}
\sigma\left(\xi_k \cdot x - t_j\right), & \hbox{if} \ 1\leq  k\leq n_q, 1 \leq j\leq 2N +3, \\
0, & \hbox{otherwise.} \end{array} \right.
\end{equation}
The number ${\mathcal N}_2$ of free parameters in the second group of $J_2 - J_1$ convolutional layers is
$$ {\mathcal N}_2 = d_{J_1} + (3s+1) \left\lceil \frac{(2N +3) d_{J_{1}}}{s-1}\right\rceil - (2s+1). $$
\end{lemma}

\begin{proof}
As in the proof of Lemma \ref{initiallayers}, we choose the bias vectors $\{b^{(j)}\}_{j=J_1 + 1}^{J_2}$ by
$$ b^{(j)} = B \left(\Pi_{p=J_1 + 1}^{j-1} \|w^{(p)}\|_1\right)  T^{(j)} {\bf 1}_{d_{j-1}} -
B \left(\Pi_{p=J_1 + 1}^{j} \|w^{(p)}\|_1\right) {\bf 1}_{d_{j-1} + s} $$
for $j=J_1 + 1, \ldots, J_2 -1$. Then we know
from \cite[Lemma 1 and Lemma 3]{zhou2020theory} again that
$$ h^{(J_{2}-1)} (x) = T^{(J_2 -1)} \ldots T^{(J_1 + 2)} T^{(J_1 + 1)} \left[\xi_\ell \cdot x\right]_{\ell=1}^{d_{J_1}} + B \left(\Pi_{p=J_1 + 1}^{J_{2}-1} \|w^{(p)}\|_1\right) {\bf 1}_{d_{J_{2}-1}} $$
and
$$ T^{(J_2)} T^{(J_2 -1)} \ldots T^{(J_1 + 2)} T^{(J_1 + 1)} = \left(W^{[1]}_{i-k}\right)_{i=1, \ldots, d_{J_{1}} + (J_{2} - J_1) s, \ k=1, \ldots, d_{J_{1}}} $$
which has entry $(i, k)$ with $i\in\{1, \ldots, d_{J_2}\}, k\in\{1, \ldots, d_{J_{1}}\}$ given by
$$ W^{[1]}_{i-k} =\left\{\begin{array}{ll}
1, & \hbox{if} \ i=k, k+ d_{J_{1}}, k+ 2 d_{J_{1}}, \ldots, k+ (2N+2)  d_{J_{1}}, \\
0, & \hbox{otherwise.} \end{array}\right. $$
It follows that the submatrix of $T^{(J_2)} T^{(J_2 -1)} \ldots T^{(J_1 + 2)} T^{(J_1 + 1)}$
consisting of the first $(2N +3)d_{J_{1}}$ rows can be expressed as
a $(2N +3) \times 1$ block matrix with each block being the $d_{J_{1}} \times d_{J_{1}}$ identity matrix.
Thus by choosing the bias vector $b^{(J_2)}$ as
\begin{eqnarray*} \left(b^{(J_2)}\right)_{i} &=& B \left(\Pi_{p=J_1 + 1}^{J_{2}-1} \|w^{(p)}\|_1\right) \left(T^{(J_2)} {\bf 1}_{d_{J_{2}-1}} \right)_{i} \\
&&+
\left\{\begin{array}{ll}
t_j, & \hbox{if} \ (j-1)d_{J_{1}} +1 \leq i \leq (j-1)d_{J_{1}} + n_q, 1\leq j \leq 2N +3, \\
B, & \hbox{otherwise,} \end{array}\right.
\end{eqnarray*}
we see that (\ref{secondlayerexplicit}) holds true.

Since the filters $\{w^{(j)}\}_{j=J_1 +1}^{J_2}$ are explicitly constructed without free parameters involved,
the free parameters required in the second group of $J_2 - J_1$ layers are those from the bias vectors $\{b^{(j)}\}_{j=J_1 +1}^{J_2}$ and the number is
$$ {\mathcal N}_2 =(2s+1) (J_2 - J_1 -1) + d_{J_2} = d_{J_1} + (3s+1) \left\lceil \frac{(2N +3) d_{J_{1}}}{s-1}\right\rceil - (2s+1), $$
which verifies the stated expression for ${\mathcal N}_2$. This proves the lemma.
\end{proof}

\subsection{Approximating $f\circ Q$ with a fully connected layer}\label{fullynet}

Recall that the layer $h^{(J_{2})}$ expressed in (\ref{secondlayerexplicit}) can be written as
$$ h^{(J_{2})} (x) =\left[H_1^T \ H_2^T \ \ldots H^T_{2N +3} \ 0 \ \ldots \ 0\right]^T, $$
where for $j=1, 2, \ldots, 2N +3$,
$$ H_j^T = \left[\sigma\left(\xi_1 \cdot x - t_j\right) \ \sigma\left(\xi_2 \cdot x - t_j\right) \ \ldots \ \sigma\left(\xi_{n_q} \cdot x - t_j\right) \ 0 \ldots 0\right] \in\RR^{d_{J_1}}. $$
Then we can construct a fully connected layer to realize $(\xi_k \cdot x)^\ell, 1\leq \ell \leq q,$ for approximating the polynomial $Q$ and then the function $f(Q(x))$.
To this end, we need the following well-known scheme of approximating univariate functions by continuous piecewise linear functions (splines)
spanned by $\{\sigma(\cdot-t_i)\}^{2N +3}_{i=1}$ with $t_i = -1+\frac{i-2}N$,
which can be found in \cite[Lemma 6]{zhou2018deepdistributed}.

\begin{lemma}\label{spline}
For $N\in\NN$, let $\textbf{t}=\{t_i: = -1+\frac{i-2}N\}_{i=1}^{2N+3}$ be the uniform mesh on $\left[-1-\frac{1}{N}, 1+\frac{1}{N}\right]$,
$L_{\textbf{t}}$ be a linear operator on $C[-1,1]$ given by
$$ L_{\textbf t}(g)(u)=\sum_{i=2}^{2N+2} g(t_i)\delta_i(u), \quad u\in [-1,1], \, g\in C[-1,1], $$
with the hat functions $\delta_i\in C(\RR)$, $i=2,\ldots, 2N+2$, given by
$$
\delta_{i}(u)=N(\sigma\left(u-t_{i-1}\right)-2\sigma\left(u-t_{i}\right)+ \sigma\left(u-t_{i+1}\right)), \qquad u\in\RR.
$$
Then for $g\in C[-1,1]$, we have $\left\|L_{\mathbf{t}}(g)\right\|_{C\left[-1,1\right]} \leq \left\|g\right\|_{C\left[-1,1\right]}$ and
$$\left\|L_{\mathbf{t}}(g)-g\right\|_{C\left[-1,1\right]} \leq 2 \omega\left(g, 1/N\right)$$
where $\omega(g, \mu)$ is the modulus of continuity of $g$ given for $\mu \in (0, 1]$ by
$$ \omega(g, \mu)=\sup_{|t|\leq \mu} \left\{|g(v)-g(v+t)|: v, v+t \in[-1,1]\right\}. $$
\end{lemma}

For convenience, we introduce a linear operator ${\mathcal L}_N: \RR^{2N+1} \to \RR^{2N+3}$ given for $\zeta =(\zeta_i)_{i=2}^{2N+2} \in \RR^{2N+1}$ by
\begin{equation}\label{diffoperator}
\left({\mathcal L}_N (\zeta)\right)_i =
\begin{cases}
\zeta_2,                            &\text{for}~i=1,\\
\zeta_3 -2\zeta_2,                &\text{for}~i=2,\\
\zeta_{i-1} -2\zeta_i +\zeta_{i+1},   &\text{for}~3\leq i \leq 2N+1,\\
\zeta_{2N+1} -2\zeta_{2N+2},        &\text{for}~i=2N+2,\\
\zeta_{2N+2},                       &\text{for}~i=2N+3.
\end{cases}
\end{equation}
It enables us to express the operator $L_{\textbf{t}}$ on $C[-1,1]$ in terms of $\{\sigma\left(\cdot-t_{j}\right)\}_{j=1}^{2N+3}$ as
$$
L_{\textbf t}(g)=N \sum_{i=1}^{2N+3} \left({\mathcal L}_N \left(\left\{g(t_k)\right\}_{k=2}^{2N+2}\right)\right)_i \sigma\left(\cdot-t_{i}\right), \qquad g\in C[-1,1].
$$
In particular, for a homogeneous polynomial $g(u) = u^\ell$ with $\ell \in\{1, \ldots, q\}$ and a vector
$v^{[\ell]} = {\mathcal L}_{N} \left(\left\{t_k^\ell\right\}_{k=2}^{2N +2}\right) \in\RR^{2 N +3}$, we have
\begin{equation}\label{hompolyapprox}
\sup_{u\in [-1, 1]} \left|N \sum_{j=1}^{2N +3} v^{[\ell]}_j \sigma\left(u-t_{j}\right) -u^\ell\right|
\leq \frac{2\ell}{N}.
\end{equation}

Denote the $n_q \times (d_{J_1} -n_q)$ zero matrix as $O$, the $n_q \times (d_{J_2} - (2N +2) d_{J_1} -n_q)$ zero matrix as $\widehat{O}$, and
\begin{equation} \label{FN1}
F^{(N)} = N \left[\begin{array}{ccccccc}
v^{[1]}_1 I_{n_q} & O & v^{[1]}_2 I_{n_q} & O & \ldots & v^{[1]}_{2N +3} I_{n_q} & \widehat{O} \\
v^{[2]}_1 I_{n_q} & O & v^{[2]}_2 I_{n_q} & O & \ldots & v^{[2]}_{2N +3} I_{n_q} & \widehat{O} \\
\vdots & \vdots & \vdots & \vdots & \vdots & \vdots & \vdots \\
v^{[q]}_1 I_{n_q} & O & v^{[q]}_2 I_{n_q} & O & \ldots & v^{[q]}_{2N +3} I_{n_q} & \widehat{O} \end{array}\right] \in \RR^{(q n_q) \times d_{J_2}}.
\end{equation}
Then
$$ F^{(N)} h^{(J_{2})} (x) =\left[P_1^T \ P_2^T \ \ldots P^T_{q}\right]^T, $$
where for $\ell=1, 2, \ldots, q$, $P_\ell$ is a vector with $n_q$ components given by
$$ P_\ell^T = \left[N \sum_{j=1}^{2N +3} v^{[\ell]}_j \sigma\left(\xi_1 \cdot x -t_{j}\right) \
\ldots \ N \sum_{j=1}^{2N +3} v^{[\ell]}_j \sigma\left(\xi_{n_q} \cdot x -t_{j}\right)\right]. $$
If we denote a row vector $\gamma$ of size $q n_q$ as
$$\gamma =\left[\gamma_{1, 1} \ \gamma_{2, 1} \ \ldots, \gamma_{n_q, 1} \ \gamma_{1, 2} \ \gamma_{2, 2} \ \ldots, \gamma_{n_q, 2} \ \ldots \
\gamma_{1, q} \ \gamma_{2, q} \ \ldots, \gamma_{n_q, q}\right], $$
then we have
$$ \gamma F^{(N)} h^{(J_{2})} (x) = \sum_{k =1}^{n_q} \sum_{\ell=1}^q \gamma_{k, \ell} N \sum_{j=1}^{2N_1 +3} v^{[\ell]}_j \sigma\left(\xi_k \cdot x -t_{j}\right)$$
and by (\ref{hompolyapprox}),
$$ \sup_{|x| \leq 1} \left|\gamma F^{(N)} h^{(J_{2})} (x) - \sum_{k =1}^{n_q} \sum_{\ell=1}^q \gamma_{k, \ell} \left(\xi_k \cdot x\right)^\ell\right| \leq
\frac{2 q \|\gamma\|_1}{N}. $$
In particular, by taking $\gamma$ to be the coefficient vector $\beta$ given in Lemma \ref{ridgepolynomialiden}, if we denote
\begin{equation}\label{Ghatdef}
\widehat{Q}(x) =\beta F^{(N)} h^{(J_{2})} (x)+ Q(0),
\end{equation}
then we have
\begin{equation}\label{Qapprox}
\sup_{|x| \leq 1} \left|\widehat{Q}(x)  - Q  (x)\right| \leq
\frac{2 q \|\beta\|_1}{N},
\end{equation}
which implies
\begin{equation}\label{Qhatnorm}
\left\|\widehat{Q}\right\|_{C(\Omega)} \leq \widehat{B}_Q :=
B_Q + 2 q \|\beta\|_1.
\end{equation}

Now we can construct the fully connected layer of width $2N +3$ to approximate the function $f(Q(x))$.

\begin{lemma}\label{ratelastfull}
Let $q, N \in\NN$, $Q$ be a polynomial on $\Omega$ of degree at most $q$, $f\in C^{0,\alpha}[-B_Q, B_Q]$ for some $0< \alpha \leq 1$,
and $h^{(J_{2})} (x)$ be given by (\ref{secondlayerexplicit}).
Then for the last layer $h^{(J_{2} +1)} (x) = \sigma\left(F^{[J_2 +1]} h^{(J_{2})} (x) - b^{(J_2 +1)}\right)$ of width $2N +3$ with
\begin{equation}\label{FNNFb}
F^{[J_2 +1]} =\left[\begin{array}{c} \beta \\
\vdots \\
\beta \end{array}\right] F^{(N)}, \qquad b^{(J_2 +1)} = - Q(0) {\bf 1}_{2N +3} + \widehat{B}_Q \left[\begin{array}{c} t_1  \\
\vdots \\
t_{2N +3}  \end{array}\right],
\end{equation}
by taking the coefficient vector $c= \frac{N}{\widehat{B}_Q} {\mathcal L}_{N} \left(\left\{f\left(\widehat{B}_Q t_k\right)\right\}_{k=2}^{2N+2}\right) \in \RR^{2N +3}$, we have
\begin{equation}\label{fullboundfQ}
\left\|\sum_{j=1}^{2N+3} c_j \left(h^{(J_{2} +1)} (x)\right)_j
- f\left(Q(x)\right)\right\|_{C(\Omega)}  \leq \frac{4 \widehat{B}_Q^{\alpha}|f|_{C^{0, \alpha}}}{N^\alpha}
\end{equation}
and
\begin{equation}\label{coefficientbound}
\|c\|_\infty \leq \frac{4\|f\|_\infty N}{\widehat{B}_Q}.
\end{equation}
The total number ${\mathcal N}_3$ of free parameters in the last layer is
$$ {\mathcal N}_3 = 2N + q n_q + 3. $$
\end{lemma}

\begin{proof}
With the stated connection matrix $F^{[J_2 +1]}$ and bias vector $b^{(J_2 +1)}$,
the last layer $h^{(J_{2} +1)} (x) = \sigma\left(F^{[J_2 +1]} h^{(J_{2})} (x) - b^{(J_2 +1)}\right)$ has components
\begin{equation}\label{sigmaQapprox}
\left(h^{(J_{2} +1)} (x)\right)_j = \sigma\left(\beta F^{(N)} h^{(J_{2})} (x) + Q(0) -\widehat{B}_Q t_{j}\right) =  \sigma\left(\widehat{Q}(x) -\widehat{B}_Q t_{j}\right)
\end{equation}
with $j=1, \ldots, 2N +3.$

Apply Lemma \ref{spline} to the function $g(u) = f\left(\widehat{B}_Q u\right)$ on $[-1, 1]$
where $f$ has been extended outside $[-B_Q, B_Q]$ as $f(u) = f(\hbox{sgn}(u)B_Q)$ for $|u| > B_Q$, which keeps the same semi-norm $|f|_{C^{0, \alpha}}$. Then
$g$ is Lipschitz-$\alpha$
with semi-norm $|g|_{C^{0, \alpha}[-1, 1]} \leq \widehat{B}_Q^{\alpha}|f|_{C^{0, \alpha}}$,
and we know with $\widehat{c}= {\mathcal L}_{N} \left(\left\{f\left(\widehat{B}_Q t_k\right)\right\}_{k=2}^{2N+2}\right)$, there holds
$$ \sup_{u\in [-1, 1]} \left|N \sum_{j=1}^{2N+3} \widehat{c}_j \sigma\left(u-t_{j}\right) - f\left(\widehat{B}_Q u\right)\right| \leq \frac{2 \widehat{B}_Q^{\alpha}|f|_{C^{0, \alpha}}}{N^\alpha}. $$
Since (\ref{Qhatnorm}) implies $u =\widehat{Q}(x)/\widehat{B}_Q \in [-1, 1]$ for all $x\in \Omega$ and
$\sigma\left(v/\widehat{B}_Q\right) = \frac{1}{\widehat{B}_Q} \sigma\left(v\right)$ for all $v\in\RR$, we have
$$ \sup_{x\in \Omega} \left|\frac{N}{\widehat{B}_Q} \sum_{j=1}^{2N+3} \widehat{c}_j \sigma\left(\widehat{Q}(x) -\widehat{B}_Q t_{j}\right) - f\left(\widehat{Q}(x)\right)\right|
\leq \frac{2 \widehat{B}_Q^{\alpha}|f|_{C^{0, \alpha}}}{N^\alpha}. $$
Combining this estimate with (\ref{sigmaQapprox}), (\ref{sigmaQapprox}), (\ref{Qapprox}) and the Lipschitz-$\alpha$ property of $f$ yields
\begin{eqnarray*}
\sup_{x\in \Omega} \left|\sum_{j=1}^{2N+3} c_j \left(h^{(J_{2} +1)} (x)\right)_j
- f\left(Q(x)\right)\right| &\leq& \frac{2 \widehat{B}_Q^{\alpha}|f|_{C^{0, \alpha}}}{N^\alpha} +
\sup_{x\in \Omega} \left|f\left(\widehat{Q}(x)\right)
- f\left(Q(x)\right)\right| \\
&\leq& \frac{2 \widehat{B}_Q^{\alpha}|f|_{C^{0, \alpha}}}{N^\alpha} +
 |f|_{C^{0, \alpha}}\left(\frac{2 q \|\beta\|_1}{N}\right)^\alpha.
\end{eqnarray*}
So the desired bound (\ref{fullboundfQ}) follows. From the expression $c=\frac{N}{\widehat{B}_Q} {\mathcal L}_{N} \left(\left\{f\left(\widehat{B}_Q t_k\right)\right\}_{k=2}^{2N+2}\right)$ of the coefficient vector, we find (\ref{coefficientbound}) is valid.

Since the matrix $F^{(N)}$ is constructed explicitly, the free parameters in the last layer are those from $\beta$, and the numbers $Q(0)$ and $\widehat{B}_Q$.
Together with the coefficients $\left\{f\left(\widehat{B}_Q t_k\right)\right\}_{k=2}^{2N+2}$, the total number of free parameters in the last layer to get the approximation
is
$$ {\mathcal N}_3 = q n_q + 2 + 2N+1 =2N + q n_q + 3. $$
This proves the lemma.
\end{proof}

\subsection{Deriving rates of approximation}

Now we can prove our first two main results.

\begin{proof}[Proof of Theorem \ref{ratefQ}]
First, we construct the first group of $J_1 = \lceil \frac{n_q d-1}{s-1}\rceil$ convolutional layers
followed by a downsampling operation as in subsection \ref{firstCNN} and get
$h^{(J_{1})} (x)$ of linear features expressed in (\ref{initialexplicit}) of
Lemma \ref{initiallayers}.

Next, we construct the second group of $J_2 - J_1 = \left\lceil \frac{(2N +3) d_{J_{1}}}{s-1}\right\rceil$ convolutional layers
as in subsection \ref{secondCNN} and get $h^{(J_{2})} (x)$ of ridge functions $\{\sigma\left(\xi_k \cdot x - t_j\right)\}_{k, j}$
expressed in (\ref{secondlayerexplicit}) of Lemma \ref{DCNNapprox}.

Then we use the last layer
$h^{(J_{2} +1)} (x) = \sigma\left(F^{[J_2 +1]} h^{(J_{2})} (x) - b^{(J_2 +1)}\right)$ of width $2N +3$
constructed in subsection \ref{fullynet} with a matrix $F^{[J_2 +1]}$ of identical rows, and know from (\ref{fullboundfQ}) of Lemma \ref{ratelastfull} that
(\ref{ratefQexplicit}) holds true with $C_{Q, \alpha} = 4 \widehat{B}_Q^{\alpha}$ and $c$ satisfying $\|c\|_\infty \leq \frac{4\|f\|_\infty N}{\widehat{B}_Q}$.

The total number ${\mathcal N}= {\mathcal N}_1 + {\mathcal N}_2 + {\mathcal N}_3$ of free parameters in this network is
$$(3s+2) J_1
+ d_{J_1} + (3s+1) \left\lceil \frac{(2N +3) d_{J_{1}}}{s-1}\right\rceil - (2s+1)
+ 2N + q n_q + 3. $$
Since $s\geq 2$ and $\lceil \frac{n}{s-1}\rceil \leq \frac{n + s-2}{s-1}$ for any integer $n\in\NN$, we see that
$$J_1 = \left\lceil \frac{n_q d-1}{s-1}\right\rceil \leq \frac{n_q d-2}{s-1} +1 $$
and
\begin{equation}\label{dJ1est}
d_{J_1}  = \left\lfloor \frac{d+ J_1 s}{d} \right\rfloor  \leq 1 + \frac{s}{d}J_1.
\end{equation}
Hence
\begin{eqnarray*}
{\mathcal N} &\leq& \left(3s+2 + \frac{s}{d} + \frac{s}{d} 7 (2N +3)\right) J_1
+ 7 (2N +3) +   s +4
+ 2N + q n_q \\
&\leq& \left(30 + 28 n_q\right) N + (8 d +q + 44) n_q + 4 s + 49.
\end{eqnarray*}

To achieve the approximation accuracy $0< \epsilon \leq 1$, by using the bound $n_q \leq d^q$ and taking
$$ N = \left\lceil C_{Q, \alpha}^{\frac{1}{\alpha}} |f|_{C^{0, \alpha}}^{\frac{1}{\alpha}} \epsilon^{-\frac{1}{\alpha}}\right\rceil, $$
we know that
$$N \leq  C_{Q, \alpha}^{\frac{1}{\alpha}} |f|_{C^{0, \alpha}}^{\frac{1}{\alpha}} \epsilon^{-\frac{1}{\alpha}} + 1 \leq \left(C_{Q, \alpha}^{\frac{1}{\alpha}} |f|_{C^{0, \alpha}}^{\frac{1}{\alpha}} +1\right)  \epsilon^{-\frac{1}{\alpha}} $$
and thereby
$$ {\mathcal N} \leq  43 \left(C_{Q, \alpha}^{\frac{1}{\alpha}} |f|_{C^{0, \alpha}}^{\frac{1}{\alpha}} +1\right) d^q \epsilon^{-\frac{1}{\alpha}} + (45+q) d^{q+1}. $$
This proves the stated complexity with the constant
$$ C_{Q, \alpha, |f|_{C^{0, \alpha}}} = 22 \left(C_{Q, \alpha}^{\frac{1}{\alpha}} |f|_{C^{0, \alpha}}^{\frac{1}{\alpha}} +1\right) + 45+q. $$
The proof of Theorem \ref{ratefQ} is complete.
\end{proof}

\begin{proof}[Proof of Theorem \ref{rateradial}]
The special radial form of the approximated function enables us to skip
the first group of convolutional layers in subsection \ref{firstCNN} by taking $J_1 =0$,
and to construct a group of $J=\left\lceil \frac{(2N +3) d}{s-1}\right\rceil$ convolutional layers
as in subsection \ref{secondCNN} by replacing $d_{J_1}$ and $n_q$ by the size $d$ of the input data $x =[x_1 \ldots x_d]^T$, and the constant $B$ by $1$.
Then $h^{(J)}$ has width $d_J = d + J s$ and we have
$$ \left(h^{(J)} (x)\right)_{(j-1)d + k} =
\left\{\begin{array}{ll}
\sigma\left(x_k  - t_j\right), & \hbox{if} \ 1\leq  k\leq d, 1 \leq j\leq 2N +3, \\
0, & \hbox{otherwise} \end{array}\right.$$
and the number ${\mathcal N}_1$ of free parameters in the convolutional layers is
$$ {\mathcal N}_1 = d + (3s+1) \left\lceil \frac{(2N +3) d}{s-1}\right\rceil - (2s+1), $$
which is similar to the expressions in Lemma \ref{DCNNapprox}.

Besides the reduction of the first groups of convolutional layers for approximating $f\circ Q$ with a general unknown feature polynomial $Q$, the second reduction
for approximating the radial function $f(|x|^2)$ is to take only one row block of the matrix $F^{(N_1)}$ corresponding to $\ell =2$ as
$$ F^{(N)} = N \left[
v^{[2]}_1 I_{d} \ \ v^{[2]}_2 I_{d} \  \ldots \ v^{[2]}_{2N +3} I_{d} \ \widehat{O} \right] \in \RR^{d \times d_J} $$
with the $d \times (d_{J} - (2N +3) d)$ zero matrix $\widehat{O}$,
and to take a row vector $\beta ={\bf 1}_d$ of size $d$ from the norm square $|x|^2 = \sum_{k =1}^{d} x_k^2$. Then
for $\widehat{Q}(x) =\beta F^{(N)} h^{(J)} (x)$, by Lemma \ref{spline} we have
\begin{equation}\label{radialquadraticapprox}
\sup_{|x| \leq 1} \left|\widehat{Q}(x) - |x|^2\right| \leq
\frac{4 d}{N} \quad \hbox{which implies} \ \left\|\widehat{Q}\right\|_{C(\Omega)} \leq \widehat{B}:=
1 + 4d.
\end{equation}

Finally we use the last layer
$h^{(J +1)} (x) = \sigma\left(F^{[J +1]} h^{(J)} (x) - b^{(J +1)}\right)$ of width $2N +3$
constructed in subsection \ref{fullynet} with
$$ F^{[J +1]} =\left[\begin{array}{c} \beta \\
\vdots \\
\beta \end{array}\right] F^{(N)}, \qquad b^{(J +1)} = \widehat{B} \left[\begin{array}{c} t_1  \\
\vdots \\
t_{2N_2 +3}  \end{array}\right]. $$
Then $\left(h^{(J +1)} (x)\right)_j = \sigma\left(\beta F^{(N)} h^{(J)} (x) -\widehat{B} t_{j}\right) =  \sigma\left(\widehat{Q}(x) -\widehat{B} t_{j}\right)$ for $j=1, \ldots, 2N +3$.

As in the proof of Lemma \ref{ratelastfull}, we apply Lemma \ref{spline} to the function $f\left(\widehat{B} \cdot\right)$ on $[-1, 1]$
where $f$ has been extended outside $[0, 1]$ as $f(u) = f(1)$ for $u>1$ and $f(u) = f(0)$ for $u<0$, which keeps the same semi-norm $|f|_{C^{0, \alpha}}$.
Then we take $\widehat{c}= {\mathcal L}_{N} \left(\left\{f\left(\widehat{B} t_k\right)\right\}_{k=2}^{2N+2}\right)$ and get
$$ \sup_{u\in [-1, 1]} \left|N \sum_{j=1}^{2N+3} \widehat{c}_j \sigma\left(u-t_{j}\right) - f\left(\widehat{B} u\right)\right| \leq \frac{2 \widehat{B}^{\alpha}|f|_{C^{0, \alpha}}}{N^\alpha}. $$
Taking $u =\widehat{Q}(x)/\widehat{B} \in [-1, 1]$ with $x\in \Omega$ yields
$$ \sup_{x\in \Omega} \left|\frac{N}{\widehat{B}} \sum_{j=1}^{2N+3} \widehat{c}_j \sigma\left(\widehat{Q}(x) -\widehat{B} t_{j}\right) - f\left(\widehat{Q}(x)\right)\right|
\leq \frac{2 \widehat{B}^{\alpha}|f|_{C^{0, \alpha}}}{N^\alpha}. $$
Combining this estimate with (\ref{radialquadraticapprox}) and the Lipschitz-$\alpha$ property of $f$ gives
\begin{eqnarray*}
\sup_{x\in \Omega} \left|\sum_{j=1}^{2N+3} \frac{N}{\widehat{B}} \widehat{c}_j \left(h^{(J_{2} +1)} (x)\right)_j
- f\left(|x|^2\right)\right| &\leq& \frac{2 \widehat{B}^{\alpha}|f|_{C^{0, \alpha}}}{N^\alpha} +
\sup_{x\in \Omega} \left|f\left(\widehat{Q}(x)\right)
- f\left(|x|^2\right)\right| \\
&\leq& \frac{2 \widehat{B}^{\alpha}|f|_{C^{0, \alpha}}}{N^\alpha} +
 |f|_{C^{0, \alpha}}\left(\frac{4d}{N}\right)^\alpha.
\end{eqnarray*}
This together with $\left\|\frac{N}{\widehat{B}} \widehat{c}\right\|_\infty \leq \frac{4N \|f\|_\infty}{1+4d}$ verifies the desired bound (\ref{rateradialbound}).

The free parameters in the last layer are only from  $\left\{f\left(\widehat{B} t_k\right)\right\}_{k=2}^{2N_2+2}$ with the total number
$$ {\mathcal N}_2 = 2N + 1. $$
Thus the total number of free parameters of our network is
$$ {\mathcal N} = {\mathcal N}_1 + {\mathcal N}_2 \leq d + 7 (2N +3) d + s + 2 N + 1 \leq (14 d +2) N + 22 d + s +1. $$

The bound for ${\mathcal N}$ when $N = \left\lceil \left(1+ 4d\right) \left(3 |f|_{C^{0, \alpha}}\right)^{\frac{1}{\alpha}} \epsilon^{-\frac{1}{\alpha}}\right\rceil$ is seen easily.
This completes the proof of Theorem \ref{rateradial}.
\end{proof}

\section{Deep Networks Approximate Much Faster}\label{mainproof}

In this section we prove out third main result which demonstrates the super efficiency of deep neural networks in approximating radial functions.

\begin{proof}[Proof of Theorem \ref{lowerboundtheorem}]
We introduce another hypothesis space spanned by ridge functions on ${\mathbb B}$ as
\begin{equation}\label{radialLip1normde}
{\mathcal R}_{N, \infty} =\left\{\sum_{k=1}^N r_k (\theta_k \cdot x): |\theta_k| =1, r_k \in L_\infty [-1, 1]\right\}.
\end{equation}
Each function from the hypothesis space ${\mathcal S}_N$ can be expressed as
$$ \sum_{k=1}^N c_k \sigma_k (a_k \cdot x -b_k) = \sum_{k=1}^N r_k (\theta_k \cdot x), $$
where the univariate function $r_k$ on $[-1, 1]$ are given by
$$ \left\{\begin{array}{ll}
r_k (u) = c_k \sigma_k (|a_k| u - b_k), \ \theta_k =\frac{1}{|a_k|} a_k, & \hbox{if} \ |a_k| \not= 0, \\
r_k (u) \equiv c_k \sigma_k (-b_k),  \ \theta_k =(1, 0, \ldots, 0)^T, & \hbox{if} \ |a_k| = 0.
\end{array}\right. $$
Note that $r_k \in L_\infty [-1, 1]$ for each $k$. Thus we find ${\mathcal S}_N \subseteq {\mathcal R}_{N, \infty}$ and
$$ \hbox{dist}\left({\mathcal B}\left(C^{0, 1}_{|\cdot|}\right), {\mathcal S}_N\right) \geq \hbox{dist}\left({\mathcal B}\left(C^{0, 1}_{|\cdot|}\right), {\mathcal R}_{N, \infty}\right). $$
Now we apply \cite[Corollary 3]{Konovalov}, with $p=q=\infty$, $r=1$ and functions extended continuously from $\{x \in\RR^d: |x| <1\}$ to ${\mathbb B}$
by the Lipschitz property, which asserts that when $d>1$,
$$ c'_d N^{-\frac{1}{d-1}} \geq \hbox{dist}\left({\mathcal B}\left(C^{0, 1}_{|\cdot|}\right), {\mathcal R}_{N, \infty}\right) \geq c_d N^{-\frac{1}{d-1}}, \qquad \forall N\in\NN, $$
where $c_d, c'_d$ are positive constants independent of $N\in\NN$. Then the first desired bound (\ref{shallowbound}) follows.

On the other hand, for each function $f(|\cdot|^2) \in {\mathcal B}\left(C^{0, 1}_{|\cdot|}\right)$ with a function $f$ on $[0, 1]$, we see that for $s\not= t \in [0, 1]$,
\begin{eqnarray*}
&&\left|f(s) - f(t)\right| = \left|f\left(|(\sqrt{s}, 0 \ldots, 0)|^2\right) - f\left(|(\sqrt{t}, 0 \ldots, 0)|^2\right)\right| \\
&&\quad \leq \|f(|\cdot|^2)\|_{C^{0, 1}({\mathbb B})} |(\sqrt{s}, 0 \ldots, 0)-(\sqrt{t}, 0 \ldots, 0)| \leq |\sqrt{s}-\sqrt{t}| \leq \sqrt{|s-t|}.
\end{eqnarray*}
It follows that $f \in C^{0, 1/2}[0, 1]$ and $|f|_{C^{0, 1/2}} \leq 1$. Thus, by Theorem \ref{rateradial} with $\alpha =\frac{1}{2}$, we have
$$ \inf_{g \in {\mathcal H}_N} \|f(|\cdot|^2)-g\|_{L_\infty({\mathbb{B}})} \leq 3 \sqrt{1+ 4d} N^{-\frac{1}{2}}. $$
This implies the second desired bound (\ref{deepbound}) and proves the theorem.
\end{proof}

\section{Generalization Analysis}
\label{generalizationerror}

In this section, we conduct generalization analysis of the deep learning algorithm induced by our constructed deep neural network.
To this end, we need to analyse the approximation ability of the hypothesis space in the ERM algorithm
by showing that the filters, the full connection matrix, and biases of the deep neural network can be bounded as required in the hypothesis space and
then to derive an estimation error bound by applying a covering number argument.

\subsection{Bounding the filters and connection matrix}\label{boundfilters}

To bound the filters of the convolutional layers and the full connection matrix for the fully connected layer,
we need the following simple consequence of the classical Cauchy' bound of polynomial roots in terms of coefficients and Vieta's formula of
polynomial coefficients in terms of roots.

\begin{lemma}\label{cauchybound}
If $w=\{w_j\}_{j\in\ZZ}$ is a real sequence supported in $\{0, \ldots, K\}$ with $w_K =1$, then all the complex roots of
its symbol $\widetilde{w}(z)=\sum_{j=0}^{K}w_j z^j$ are located in the disk of radius $1 + \max_{j=0, \ldots, K-1}|w_j|$, the Cauchy bound of $\widetilde{w}$.

If we factorize $\widetilde{w}$ into monic polynomials of degree at most $s$, then all the coefficients of these factor polynomials are bounded by
$s^{\frac{s}{2}} \left(1 + \max_{j=0, \ldots, K-1}|w_j|\right)^s \leq  s^{\frac{s}{2}} \left(1+\|w\|_\infty\right)^{s}$.
\end{lemma}

\begin{lemma}\label{boundweightslemma}
Let $2 \leq s \leq d$, $q, N \in\NN$, $Q$ be a polynomial on $\Omega$ of degree at most $q$, $f\in C[-B_Q, B_Q]$.
Then for the deep neural network constructed in Section \ref{construct}, there exists a constant
$R= R_{q, d, s, Q, \|f\|_\infty} \geq |Q(0)|+2B_Q$ depending on $q, d, s, Q, \|f\|_\infty$ such that
$$ \left\|w^{(j)}\right\|_\infty \leq R, \qquad j=1, \ldots, J_2, $$
and
$$ \left\|F^{[J_2+1]}\right\|_\infty \leq N^2 R, \qquad \|c\|_\infty \leq N R. $$
\end{lemma}

\begin{proof}
Since $|\xi_{n_q}| =1$, there exists some $\ell \in\{1, \ldots, d\}$ such that $\left(\xi_{n_q}\right)_\ell \not=0$ and $\left(\xi_{n_q}\right)_i =0$ for any $i< \ell$.
Then we see that the sequence $W$ constructed in the proof of Lemma \ref{initiallayers} is supported in $\{0, \ldots, n_q d -\ell\}$
with $W_{n_q d -\ell} = \left(\xi_{n_q}\right)_\ell \not= 0$.
Set a sequence $w = \frac{1}{\left(\xi_{n_q}\right)_\ell} W$. Then $w$ satisfies the condition in Lemma \ref{cauchybound} with $K=n_q d -\ell$.
So by Lemma \ref{cauchybound}, all the complex roots of $\widetilde{w}$ are located in the disk of radius
$1 + \|w\|_\infty = 1 + \frac{1}{\left|\left(\xi_{n_q}\right)_\ell\right|} \max_{k=1, \ldots, n_q} \|\xi_k\|_\infty \leq 1 + \frac{1}{\left|\left(\xi_{n_q}\right)_\ell\right|}$,
and the filters $\{w^{(j)}\}_{j=1}^{J_1}$ satisfying (\ref{Wk}) and (\ref{firstlayerfactor}) constructed in Lemma \ref{initiallayers} can be bounded as
$$ \left\|w^{(j)}\right\|_\infty \leq s^{s/2} \left(1 + \left|\left(\xi_{n_q}\right)_\ell\right|\right)  \left(1 + \frac{1}{\left|\left(\xi_{n_q}\right)_\ell\right|}\right)^s, \qquad j=1, \ldots, J_1. $$

For the second layer of CNNs, we observe that the sequence $W^{[1]}$ satisfies the condition in Lemma \ref{cauchybound} with $K=(2N+3) d_{J_1}$ and $\left\|W^{[1]}\right\|_\infty =1$.
Then by Lemma \ref{cauchybound}, the filters $\{w^{(j)}\}_{j=J_1+1}^{J_2}$ defined by (\ref{Wsequencesecond}) and (\ref{filterW1}) can be bounded as
$$\|w^{(j)}\|_\infty \leq  s^{\frac{s}{2}} 2^s, \qquad j=J_1+1, \ldots, J_2. $$

For the connection matrix $F^{[J_2+1]}$ of the fully connected layer defined by (\ref{FN1}) and (\ref{FNNFb}), from the bound $\|v^{[\ell]}\|_\infty \leq 4$, we know that the $\ell_1$-norm of each row of $F^{[J_2+1]}$ is bounded by $4N(2N+3) \|\beta\|_1$. Hence $\left\|F^{[J_2+1]}\right\|_\infty \leq 4N(2N+3) \|\beta\|_1$.

The above estimates together with (\ref{coefficientbound}) verify the desired bounds with the constant $R= R_{q, d, s, Q, \|f\|_\infty}$ depending on $q, d, s, Q, \|f\|_\infty$ given explicitly by
$$R= \max\left\{s^{s/2} \left(1 + \left|\left(\xi_{n_q}\right)_\ell\right|\right)  \left(1 + \frac{1}{\left|\left(\xi_{n_q}\right)_\ell\right|}\right)^s,
s^{\frac{s}{2}} 2^s, 20 \|\beta\|_1, \frac{4\|f\|_\infty}{\widehat{B}_Q}, |Q(0)|+2B_Q\right\}. $$
This proves the lemma.
\end{proof}

\subsection{Bounding the biases}\label{boundbiases}

Applying the bounds for the filters, we can bound the biases as follows.

\begin{lemma}\label{boundbiaseslemma}
In the setting of Lemma \ref{boundweightslemma},
for the deep neural network constructed in Section \ref{construct}, with the constant $R= R_{q, d, s, Q, \|f\|_\infty}$ given in Lemma \ref{boundweightslemma} there holds
$$\left\|b^{(j)}\right\|_\infty \leq \left(2(s+1) R\right)^j, \qquad j=1, \ldots, J_2+1. $$
\end{lemma}

\begin{proof}
The bias vectors $\{b^{(j)}\}_{j=1}^{J_1}$ of the first group of CNNs are chosen in the proof of Lemma \ref{initiallayers} as
$b^{(1)} =  - \|w^{(1)}\|_1 {\bf 1}_{d_{1}}$ and $b^{(j)} = \left(\Pi_{p=1}^{j-1} \|w^{(p)}\|_1\right)  T^{(j)} {\bf 1}_{d_{j-1}} - \left(\Pi_{p=1}^{j} \|w^{(p)}\|_1\right) {\bf 1}_{d_{j-1} + s}$, for $j=2, \ldots, J_1$. By the special structure of Toeplitz matrix of the convolutional filters and Lemma \ref{boundweightslemma},
we find $\|b^{(j)}\|_\infty \leq 2\left((s+1)R\right)^j$ for $j=1,2, \ldots, J_1$.

The bias vectors $\{b^{(j)}\}_{j=J_1+1}^{J_2-1}$ of the second group of CNNs are chosen in the proof of lemma \ref{DCNNapprox} to be
$b^{(j)} = B \left(\Pi_{p=J_1 + 1}^{j-1} \|w^{(p)}\|_1\right)  T^{(j)} {\bf 1}_{d_{j-1}} - B \left(\Pi_{p=J_1 + 1}^{j} \|w^{(p)}\|_1\right) {\bf 1}_{d_{j-1} + s} $,
where $B= \Pi_{p=1}^{J_{1}} \|w^{(p)}\|_1 \leq \left((s+1)R\right)^{J_1}$. Hence we also have
$\|b^{(j)}\|_\infty \leq 2\left((s+1)R\right)^j$ for $j=J_1 + 1, \ldots, J_2 -1$.

The bias vector in the $J_2$-th layer is given in the proof of Lemma \ref{DCNNapprox} by
\begin{eqnarray*} \left(b^{(J_2)}\right)_{i} &=& B \left(\Pi_{p=J_1 + 1}^{J_{2}-1} \|w^{(p)}\|_1\right) \left(T^{(J_2)} {\bf 1}_{d_{J_{2}-1}} \right)_{i} \\
&&+
\left\{\begin{array}{ll}
t_j, & \hbox{if} \ (j-2)d_{J_{1}} +1 \leq i \leq (j-2)d_{J_{1}} + n_q, 1\leq j \leq 2N_1 +3, \\
B, & \hbox{otherwise.} \end{array}\right.
\end{eqnarray*}
As $|t_j| \leq 2$, we also have $\|b^{(j)}\|_\infty \leq 2\left((s+1)R\right)^j$ for $j=J_2$.

Finally, the bias vector in the $J_2+1$-th layer is given in Lemma  \ref{ratelastfull} as $b^{(J_2 +1)} = - Q(0) {\bf 1}_{2N_2 +3} + B_Q \left[t_1, \dots, t_{2N_2 +3} \right]^T$.
Its entries can be bounded by $|Q(0)|+2B_Q$, a constant depending only on $Q$. That is, $\|b^{(j)}\|_\infty \leq |Q(0)|+2B_Q$ for $j=J_2+1$.
But $|Q(0)|+2B_Q \leq R$. Then the desired bounds holds.
The proof of the lemma is complete.
\end{proof}

\subsection{Bounding covering numbers}
\label{coveringnumber}

Recall that the {\bf covering number} $\mathcal{N}\left(\eta, {\mathcal H}\right)$ of a compact subset ${\mathcal H}$ of $C(\Omega)$ is defined for $\eta>0$ to be the
smallest integer $\ell$ such that ${\mathcal H}$ is contained in the union of $\ell$ balls in $C(\Omega)$ of radius $\eta$.
Covering numbers can be used to measure the capacity of a hypothesis space and hence the learning performance of the induced ERM algorithms.
For our generalization analysis, we need to estimate the covering numbers of the bounded hypothesis space ${\mathcal H}_{R, N}$.

\begin{lemma}\label{coveringnumberlemma}
For $R \geq 2, N\in\NN$, with two constants $A_1, A_2$ depending only on $d, q, s$, there holds
$$ \log \mathcal{N}\left(\eta, {\mathcal H}_{R, N}\right) \leq  A_1 N \log \left(2/\eta\right)
+ A_2 N^2 \log \left(2(s+1)R\right), \qquad \forall \ 0<\eta \leq 1.
$$
\end{lemma}

\begin{proof}
For a vector $h=(h_i(x))_{i=1}^K$ of functions on $\Omega$, denote $\|h\|_\infty = \max_{i=1, \ldots, K} \|h_i\|_\infty$.

If ${\bf w}, {\bf b}, F^{[J_2 +1]}$ and $c$ satisfy the restrictions in (\ref{hypothesisRN}), then from the iteration relation (\ref{dcnn}),
the linear increment of ReLU $\sigma(u) \leq |u|$, and the special form of the rows of the Toeplitz type matrix $T^{(j)}$, we have
$$
\|h^{(j)}(x)\|_{\infty} \leq (s+1)R\|h^{(j-1)}(x)\|_{\infty}+ \left(2(s+1) R\right)^j, \qquad j=1, \ldots, J_2,
$$
which together with the input bound $\|h^{(0)}\|_\infty \leq 1$ implies by induction
\begin{equation}\label{boundforsequence}
\|h^{(j)}(x)\|_{\infty} \leq 2 \left(2(s+1) R\right)^j, \qquad j=1, \ldots, J_2,
\end{equation}
and
\begin{equation}\label{boundforsequencefull}
\|h^{(J_2 +1)}(x)\|_{\infty} \leq  N^2 \left(2(s+1) R\right)^{J_2 +1}.
\end{equation}

If $\widehat{c} \cdot \widehat{h}^{(J_2+1)}(x)$ is another function from the hypothesis space ${\mathcal H}_{R, N}$ induced by
$\widehat{{\bf w}}, \widehat{{\bf b}}, \widehat{F}^{[J_2 +1]}$, $\widehat{c}$ satisfying the restrictions in (\ref{hypothesisRN}) and
$$ \|w^{(j)}-\widehat{w}^{(j)}\|_\infty \leq \eta, \ \|b^{(j)}-\widehat{b}^{(j)}\|_\infty \leq \eta, \ \|c-\widehat{c}\|_\infty \leq \eta,
\ \|F^{[J_2+1]}-\widehat{F}^{[J_2+1]}\|_\infty \leq \eta, $$
then for $j=1, \ldots, J_2$, by the Lipschitz property of ReLU, we have
\begin{eqnarray*}
&&\|h^{(j)} -\widehat{h}^{(j)}\|_\infty =
\left\|\sigma\left(T^{w^{(j)}} h^{(j-1)}(x) - b^{(j)}\right) - \sigma\left(T^{\widehat{w}^{(j)}} \widehat{h}^{(j-1)}(x) - \widehat{b}^{(j)}\right)\right\|_\infty \\
&\leq& \left\|\left(T^{w^{(j)}} h^{(j-1)}(x) - b^{(j)}\right) - \left(T^{\widehat{w}^{(j)}} \widehat{h}^{(j-1)}(x) - \widehat{b}^{(j)}\right)\right\|_\infty \\
&\leq& \left\|T^{w^{(j)}} \left(h^{(j-1)}(x) - \widehat{h}^{(j-1)}(x)\right)\right\|_\infty +
\left\|\left(T^{w^{(j)}} - T^{\widehat{w}^{(j)}}\right) \widehat{h}^{(j-1)}(x)\right\|_\infty +  \eta.
\end{eqnarray*}
Combining this with (\ref{boundforsequence}) and the special form of the rows of the Toeplitz type matrices $T^{w^{(j)}}$ and
$T^{w^{(j)}} - T^{\widehat{w}^{(j)}} = T^{w^{(j)}-\widehat{w}^{(j)}}$, we find
$$ \|h^{(j)} -\widehat{h}^{(j)}\|_\infty \leq (s+1)R \left\|h^{(j-1)} - \widehat{h}^{(j-1)}\right\|_\infty +
(s+1) \eta  2 \left(2(s+1) R\right)^{j-1}+  \eta. $$
This together with the fact $R \geq 2$ and $h^{(0)} -\widehat{h}^{(0)}=0$ implies by induction
$$ \|h^{(j)} -\widehat{h}^{(j)}\|_\infty \leq
2 \left(2(s+1) R\right)^{j} \eta, \qquad j=1, \ldots, J_2. $$

In the same way, for the fully connected layer, we know that $\|h^{(J_2+1)} -\widehat{h}^{(J_2+1)}\|_\infty$ is bounded by
\begin{eqnarray*}
&&
\left\|\left(F^{[J_2+1]} h^{(J_2)}(x) - b^{(J_2+1)}\right) - \left(\widehat{F}^{[J_2+1]} \widehat{h}^{(J_2)}(x) - \widehat{b}^{(J_2+1)}\right)\right\|_\infty \\
&\leq& \left\|F^{[J_2+1]} \left(h^{(J_2)}(x) - \widehat{h}^{(J_2)}(x)\right)\right\|_\infty +
\left\|\left(F^{[J_2+1]} -\widehat{F}^{[J_2+1]}\right) \widehat{h}^{(J_2)}(x)\right\|_\infty +  \eta \\
&\leq& N^2 R 2 \left(2(s+1) R\right)^{J_2} \eta + \eta 2 \left(2(s+1) R\right)^{J_2} + \eta
\leq N^2 \left(2(s+1) R\right)^{J_2 +1} \eta.
\end{eqnarray*}
Combining this with (\ref{boundforsequencefull}) yields
\begin{eqnarray*}
&&\|c \cdot h^{(J_2+1)} -\widehat{c} \cdot\widehat{h}^{(J_2+1)}\|_\infty \leq
\left\|c \cdot \left(h^{(J_2+1)} -\widehat{h}^{(J_2+1)}\right)\right\|_\infty +  \left\|\left(c -\widehat{c}\right) \cdot \widehat{h}^{(J_2+1)}\right\|_\infty \\
&\leq& (2N+3) N R \left\|h^{(J_2+1)} -\widehat{h}^{(J_2+1)}\right\|_\infty +  (2N+3) \eta \left\|\widehat{h}^{(J_2+1)}\right\|_\infty \\
&\leq& (2N+3) N R N^2 \left(2(s+1) R\right)^{J_2 +1} \eta + (2N+3) \eta N^2 \left(2(s+1) R\right)^{J_2 +1} \\
&\leq& N^2 (2N+3) (N+1) R \left(2(s+1) R\right)^{J_2 +1} \eta.
\end{eqnarray*}
Recall that $J_1 \leq \frac{n_q d-2}{s-1} +1 \leq n_q d -1$ and the bound $d_{J_1} \leq 1 + \frac{s}{d}J_1$ in (\ref{dJ1est}).
Then
$$  d_{J_1} \leq 1 + \frac{s}{d}\left(\frac{n_q d-2}{s-1} +1\right) \leq 2 + \frac{s}{s-1}\frac{n_q d-2}{d} < 2 + 2 n_q. $$
Hence $d_{J_1} \leq 2 n_q +1$. This together with the definition of $J_2$ yields
$$ J_2  < J_1 + \left( \frac{(2N +3) d_{J_{1}}}{s-1} +1\right) \leq n_q d + (2N +3)(1 + 2 n_q)$$
and thereby
$$ J_2 \leq \left(4 + (10+d) n_q\right) N. $$
Also,
$$ d_{J_2} \leq d_{J_1} + s \left( \frac{(2N +3) d_{J_{1}} + s-2}{s-1}\right) < (4N +7)d_{J_{1}} +s-1 $$
which gives
$$ d_{J_2} \leq (2 n_q+1)(s+9)N. $$
Therefore,
$$\|c \cdot h^{(J_2+1)} -\widehat{c} \cdot\widehat{h}^{(J_2+1)}\|_\infty \leq 10 R N^4
\left(2(s+1) R\right)^{\left(5 + (10+d) n_q\right) N} \eta =: \widehat{\eta}.
$$
Thus, by taking an $\eta$-net for each of $w^{(j)}, b^{(j)}, c, F^{[J_2+1]}$,
we know that the covering number of the hypothesis space ${\mathcal H}_{R, N}$ with the radius $\widehat{\eta} \in (0, 1]$ can be bounded as
\begin{eqnarray*}
&\mathcal{N}\left(\widehat{\eta}, {\mathcal H}_{R, N}\right) \leq \left\lceil 2R/\eta \right\rceil^{(s+1) J_2} \Pi_{j=1}^{J_2 -1} \left\lceil 2(2(s+1)R)^j/\eta \right\rceil^{2s+1}
\left\lceil 2(2(s+1)R)^{J_2}/\eta \right\rceil^{d_{J_2}} \\
& \qquad \left\lceil 2(2(s+1)R)^{J_2 +1}/\eta \right\rceil^{2N+3}
\left\lceil 2 N R/\eta \right\rceil^{2N+3} \left\lceil  2 N^2 R/\eta \right\rceil^{d_{J_2}} \\
&\leq \left(3R/\eta\right)^{(3s+2) J_2 + 4N+6+ 2d_{J_2}} N^{2N+3+ 2 d_{J_2}} \left(2(s+1)R\right)^{(2s+1)\frac{J_2 (J_2 -1)}{2} + J_2 d_{J_2} + (J_2 +1)(2N+3)}.
\end{eqnarray*}
It follows that
$$ \mathcal{N}\left(\widehat{\eta}, {\mathcal H}_{R, N}\right) \leq \left(30 R^2/\widehat{\eta}\right)^{A_1 N} N^{6 A_1 N}
\left(2(s+1)R\right)^{A'_2 N^2}, $$
where
$$ A_1 = (3s+2) \left(4 + (10+d) n_q\right) + 10 + 2 (2 n_q+1)(s+9), \quad $$
and
$$ A'_2 =\left(5 + (10+d) n_q\right) A_1 + (s+9) \left(4 + (10+d) n_q\right)\left(5 + (12+d) n_q\right). $$
But $15 R^2 \leq \left(2(s+1)R\right)^{2}$, $\log N \leq N$, and $A'_2 \geq 8 A_1$. So we have
$$ \log \mathcal{N}\left(\widehat{\eta}, {\mathcal H}_{R, N}\right) \leq A_1 N \log \left(2/\widehat{\eta}\right)
+ 2 A'_2 N^2 \log \left(2(s+1)R\right). $$
This verifies the desired bound for the covering numbers with $A_2 = 2 A'_2$ and completes the proof of the lemma.
\end{proof}

\subsection{Deriving learning rates of the ERM algorithm}
\label{proofoftheorem4}

The proof of Theorem \ref{generalizationerrortheorem} follows from our approximation error estimate in Theorem \ref{ratefQ} and
the following general inequality for the ERM algorithm
$$
f_{D, \mathcal{H}}=\arg\min_{f\in\mathcal{H}}\frac{1}{m}\sum_{i=1}^{m}(f(x_i)-y_i)^2
$$
over a compact subset $\mathcal{H}$ of $C(\mathcal{X})$, which can be verified with the same proof as that of
\cite[Theorem 2]{ChuiLinZhou20192}.

\begin{lemma} \label{generalizationerrorlemma}
Suppose there exist constants $n', \mathcal{U}, n''>0$ such that
\begin{equation}\label{covercondition}
\log\mathcal{N}(\epsilon,\mathcal{H})\leq n'\log\frac{\mathcal{U}}{\epsilon} + n'', \qquad \forall\epsilon>0.
\end{equation}
Then for any $h^*\in\mathcal{H}$ and $\epsilon>0$,
\begin{equation}
\begin{aligned}
&\hbox{Prob} \left\{\|\pi_{M}f_{D,\mathcal{H}}-f_\rho\|_\rho^2>\epsilon+2\|h^*-f_\rho\|_\rho^2\right\}\leq \exp{ \left\lbrace  n'\log\frac{16\mathcal{U} M}{\epsilon} + n'' -\frac{3m\epsilon}{512 M^2} \right\rbrace} \\ &+  \exp{ \left\lbrace \frac{-3m\epsilon^2}{16(3 M+\|h^*\|_{L_\infty(\mathcal{X})})^2(6\| h^*-f_\rho \|_\rho^2+\epsilon)} \right\rbrace }
\end{aligned}
\end{equation}
\end{lemma}

We are in a position to prove our last main result.

\begin{proof}[Proof of Theorem \ref{generalizationerrortheorem}]
Let $R \geq R_{q, d, s, Q, \|f\|_\infty}$ with the constant $R_{q, d, s, Q, \|f\|_\infty}$ given in Lemma \ref{boundweightslemma}.
By Lemmas \ref{boundweightslemma} and \ref{boundbiaseslemma}, we know from
Theorem \ref{ratefQ} that there exists some $h \in {\mathcal H}_{R, N}$ and $C^*_1 := C_{Q, \alpha} |f_\rho|_{C^{0, \alpha}}$ such that
$$\|h-f_\rho\|_\rho \leq \|h-f_\rho\|_{C(\Omega)}\leq C_1^* N^{-\alpha}. $$
Since $|y|\leq M$ almost surely, $\|f_\rho\|_{C(\Omega)} \leq M$, then
$$\|h\|_{C(\Omega)} \leq M+C_1^*.$$

According to Lemma \ref{coveringnumberlemma}, we know that (\ref{covercondition}) holds true for $\mathcal{H} = \mathcal{H}_{R, N}$ with
$n'=A_1 N$, $\mathcal{U}=2$, and $n'' = A_3 N^2$ where we denote $A_3 = A_2 \log \left(2(s+1)R\right)$.
Applying Lemma \ref{generalizationerrorlemma} to this hypothesis space and $h^* = h$, we see that for any $\epsilon>0$,
\begin{eqnarray*}
&\hbox{Prob} \left\{\|\pi_M f_{D, R, N}-f_\rho\|_\rho^2 >\epsilon+2\|h-f_\rho\|_\rho^2\right\} \\
&\leq \exp{ \left\lbrace  A_1 N \log\frac{32 M}{\epsilon} + A_3 N^2  -\frac{3m\epsilon}{512 M^2} \right\rbrace} +
\exp{ \left\lbrace \frac{-3m\epsilon^2}{16(4 M+C_1^*)^2(6 (C_1^*)^2 N^{-2\alpha}+\epsilon)} \right\rbrace }.
\end{eqnarray*}
If we restrict
\begin{equation}\label{epsiloncondition}
\epsilon\geq 6 {C_1^*}^2 N^{-2\alpha}
\end{equation}
and apply $\log N \leq N$, then we have
$$\hbox{Prob} \left\{\|\pi_M f_{D, R, N}-f_\rho\|_\rho^2 >2 \epsilon\right\}
\leq \exp{ \left\lbrace  A_4 N^2  -\frac{3m\epsilon}{512 M^2} \right\rbrace} +
\exp{ \left\lbrace \frac{-3m\epsilon}{32(4 M+C_1^*)^2} \right\rbrace }, $$
where
$$ A_4 = A_1 \log\frac{32 M}{6 {C_1^*}^2} + 2 A_1 + A_3. $$
If we set a further restriction condition on $\epsilon$ as
\begin{equation}\label{epsilonfurthercondition}
\epsilon\geq \frac{1024 M^2 A_4 N^2}{3 m},
\end{equation}
then we have
$$
\hbox{Prob} \left\{\|\pi_M f_{D, R, N}-f_\rho\|_\rho^2 >2 \epsilon\right\}
\leq \exp{ \left\lbrace  -\frac{3m\epsilon}{1024 M^2} \right\rbrace} +
\exp{ \left\lbrace \frac{-3m\epsilon}{32(4 M+C_1^*)^2} \right\rbrace }.
$$
By setting $t=2\epsilon$ and
$$ A_5 := \max\left\{12 {C_1^*}^2, 683 M^2 A_4, \ 22(4 M+C_1^*)^2\right\}, $$
this yields
\begin{equation}\label{boundprob}
\hbox{Prob} \left\{\|\pi_M f_{D, R, N}-f_\rho\|_\rho^2 > t\right\}
\leq 2 \exp{ \left\lbrace  -\frac{m t}{A_5} \right\rbrace}, \qquad \forall t \geq A_5 \max\left\{N^{-2\alpha}, \ \frac{N^2}{m}\right\}.
\end{equation}
It follows that for any $0< \delta <1$, with confidence at least $1-\delta$, there holds
$$
\|\pi_M f_{D, R, N}-f_\rho\|_\rho^2 \leq A_5 \max\left\{N^{-2\alpha}, \ \frac{N^2}{m}, \ \frac{\log (2/\delta)}{m} \right\}.
$$

If we apply the formula for the mean of the non-negative random variables $\xi=\|\pi_M f_{D, R, N}-f_\rho\|_\rho^2$:
$$
\mathbb{E}[\xi]=\int_0^{\infty} \hbox{Prob} \left[\xi>t\right]dt,
$$
we see from (\ref{boundprob}) that with $\Delta := A_5 \max\left\{N^{-2\alpha}, \ \frac{N^2}{m}\right\}$, there holds
\begin{equation*}
\begin{aligned}
\mathbb{E}\left[\|\pi_M f_{D, R, N}-f_\rho\|_\rho^2\right] &= \left(\int_0^{\Delta}+\int_{\Delta}^{\infty} \right)  \hbox{Prob} \left\{\|\pi_M f_{D, R, N}-f_\rho\|_\rho^2 > t\right\} dt \\
& \leq \Delta + \int_{\Delta}^{\infty} 2 \exp{ \left\lbrace  -\frac{m t}{A_5} \right\rbrace} d t
\leq \Delta + \frac{2 A_5}{m} \leq 3 \Delta.
\end{aligned}
\end{equation*}
Observe that
$$ 3 A_5 \leq C_{Q, s, d, \alpha, M, |f|_{C^{0, \alpha}}} \log \left(2(s+1)R\right), $$
where $C_{Q, s, d, \alpha, M, |f|_{C^{0, \alpha}}}$ is a constant independent of $m, N$ or $R$ given by
$$ 3 \max\left\{12 {C_1^*}^2, 683 M^2 \left(A_1 \log\frac{32 M}{6 {C_1^*}^2} + 2 A_1 + A_2\right), \ 22(4 M+C_1^*)^2\right\}. $$
Then the desired error bound follows. The proof of Theorem \ref{generalizationerrortheorem} is complete.
\end{proof}

\section*{Acknowledgments}

We thank the anonymous referees for their constructive suggestions. 
The last author is supported partially by the Research Grants Council of Hong Kong [Project \# CityU 11307319], Hong Kong Institute for Data Science, 
InnoHK, and National Natural Science Foundation of China [Project No. 12061160462]. 
The paper was revised when the last author visited SAMSI/Duke during his sabbatical leave. He would like to express his gratitude to their hospitality and financial support.

\bibliographystyle{abbrvnat}

\end{document}